\documentclass{article}
\usepackage{iclr2021_conference,times}

\usepackage{amsmath,amsfonts,bm}

\def\eqref#1{eq.~(\ref{#1})}
\def\Eqref#1{Eq.~(\ref{#1})}

\def\1{\bm{1}}

\DeclareMathAlphabet{\mathsfit}{\encodingdefault}{\sfdefault}{m}{sl}
\SetMathAlphabet{\mathsfit}{bold}{\encodingdefault}{\sfdefault}{bx}{n}

\DeclareMathOperator*{\argmin}{arg\,min}

\usepackage{microtype}
\usepackage{graphicx}
\usepackage{booktabs} 
\usepackage{url}
\usepackage{mathtools, enumerate}
\mathtoolsset{
showonlyrefs=false 
}

\usepackage{cancel, algorithm, algorithmic}
\usepackage{comment}
\usepackage[colorinlistoftodos]{todonotes}
\usepackage[normalem]{ulem}
\marginparwidth=\dimexpr \marginparwidth + 3cm\relax

\DeclareMathOperator*{\expec}{\mathbb E}

\newcommand{\feat}{f}
\newcommand{\target}{h}
\newcommand{\critic}{g}

\newcommand{\nrm}[1]{\Vert {#1} \Vert_2^2}

\newtheorem{thm*}{Theorem}
\newtheorem{cor*}[thm*]{Corollary}
\newtheorem{prop*}[thm*]{Proposition}
\newtheorem{lem*}[thm*]{Lemma}

\usepackage{hyperref}
\usepackage{url}

\usepackage{enumerate}
\usepackage{tikz}
\usetikzlibrary{positioning}
\usetikzlibrary{fit}
\usepackage{amsfonts,amsmath,amssymb,amsthm}
\usepackage{marvosym}

\newcommand{\bigCI}{\mathrel{\text{\scalebox{1.07}{$\perp\mkern-10mu\perp$}}}}

\newcommand{\relic}{\textsc{ReLIC}}

\usepackage{xargs}                    
\usepackage{wrapfig}
\usepackage[position=b]{subcaption}
\usetikzlibrary{matrix}
\usetikzlibrary{positioning}
\usetikzlibrary{graphs}
\usetikzlibrary{arrows}
\usetikzlibrary{shapes}
\usetikzlibrary{patterns,decorations.pathreplacing}

\usepackage{amsmath, amsthm, bbm, amsfonts, amssymb, cite}  

\newtheorem{theorem}{Theorem}
\newtheorem{theorem*}{Theorem}
\newtheorem{lemma}{Lemma}
\newtheorem{definition}{Definition}
\newenvironment{oneshot}[1]{\begin{theorem*}{#1}{\unskip}}{\end{theorem*}}

\title{Representation Learning via \\ Invariant Causal Mechanisms}

\author{%
  Jovana Mitrovic
  \quad Brian McWilliams
  \quad Jacob Walker
  \quad Lars Buesing
  \quad Charles Blundell \\ \\
    \hspace{5cm} DeepMind, London, UK \\
    \hspace{0.6cm} \texttt{\{mitrovic, bmcw, jcwalker, lbuesing, cblundell\}@google.com} \\
}

\iclrfinalcopy
\begin{document}

\maketitle

\begin{abstract}
Self-supervised learning has emerged as a strategy to reduce the reliance on costly supervised signal by pretraining representations only using unlabeled data.
These methods combine heuristic proxy classification tasks with data augmentations and have achieved significant success, but our theoretical understanding of this success remains limited.
In this paper we analyze self-supervised representation learning using a causal framework.
We show how data augmentations can be more effectively utilized through explicit \emph{invariance constraints} on the proxy classifiers employed during pretraining. 
Based on this, we propose a novel self-supervised objective, 
Representation Learning via Invariant Causal Mechanisms (\relic{}), that enforces \emph{invariant prediction} of proxy targets across augmentations through an invariance regularizer
which yields improved generalization guarantees.
Further, using causality we generalize contrastive learning, a particular kind of self-supervised method, and provide an alternative theoretical explanation for the success of these methods.
Empirically, \relic{} significantly outperforms competing methods in terms of robustness and out-of-distribution generalization on ImageNet, while also significantly outperforming these methods on Atari achieving above human-level performance on $51$ out of $57$ games. 
\end{abstract}

\section{Introduction}

Training deep networks often relies heavily on large amounts of useful supervisory signal, such as labels for supervised learning or rewards for reinforcement learning. 
These training signals can be costly or otherwise impractical to acquire. 
On the other hand, unsupervised data is often abundantly available. 
Therefore, pretraining representations for unknown downstream tasks without the need for labels or extrinsic reward holds great promise for reducing the cost of applying machine learning models.
To pretrain representations, self-supervised learning makes use of proxy tasks defined on unsupervised data.
Recently, self-supervised methods using contrastive objectives have emerged as one of the most successful strategies for unsupervised representation learning \citep{oord2018representation, hjelm2018learning, chen2020simple}.
These methods learn a representation by classifying every datapoint against all others datapoints (negative examples). 
Under assumptions on how the negative examples are sampled, minimizing the resulting contrastive loss has been justified as maximizing a lower bound on the mutual information (MI) between representations \citep{poole2019variational}. 
However, \citep{tschannen2019mutual} has shown that performance on downstream tasks may be more tightly correlated with the choice of encoder architecture than the achieved MI bound, highlighting issues with the MI theory of contrastive learning. 
Further, contrastive approaches compare different views of the data (usually under different data augmentations) to calculate similarity scores.
This approach to computing scores has been empirically observed as a key success factor of contrastive methods, but has yet to be theoretically justified.
This lack of a solid theoretical explanation for the effectiveness of contrastive methods hinders their further development.

To remedy the theoretical shortcomings, we analyze the problem of self-supervised representation learning through a causal lens.
We formalize intuitions about the data generating process using a causal graph and leverage causal tools to derive properties of the optimal representation.
We show that a representation should be an \emph{invariant predictor} of proxy targets under interventions on features that are only correlated, but not causally related to the downstream targets of interest.
Since neither causally nor purely correlationally related features are observed and thus performing actual interventions on them is not feasible, for learning representation with this property we use data augmentations to simulate a subset of possible interventions.
Based on our causal interpretation, we propose a regularizer which enforces that the prediction of the proxy targets is invariant across data augmentations. 
We propose a novel objective for self-supervised representation learning called REpresentation Learning with Invariant Causal mechanisms (\relic). 
We show how this explicit invariance regularization leverages augmentations more effectively than previous self-supervised methods and that representations learned using \relic{} are guaranteed to generalize well to downstream tasks under weaker assumptions than those required by previous work \citep{saunshi2019theoretical}.

Next we generalize contrastive learning and provide an alternative theoretical explanation to MI for the success of these methods.
We generalize the proxy task of instance discrimination commonly used in contrastive learning using the causal concept of \emph{refinements} \citep{chalupka2014visual}. 
Intuitively, a refinement of a task can be understood as a more fine-grained variant of the original problem. 
For example, a refinement for classifying cats against dogs would be the task of classifying individual cat and dog breeds. 
The instance discrimination task results from the most fine-grained refinement, e.g. discriminating individual cats and dogs from one another.
We show that using refinements as proxy tasks enables us to learn useful representations for downstream tasks.
Specifically, using causal tools, we show that learning a representation on refinements such that it is an invariant predictor of proxy targets across augmentations is a \emph{sufficient condition} for these representations to generalize to downstream tasks (cf. Theorem \ref{thm.condition}). 
In summary, we provide theoretical support both for the general form of the contrastive objective as well as for the use of data augmentations. 
Thus, we provide an alternative explanation to mutual information for the success of recent contrastive approaches namely that of causal refinements of downstream tasks. 

We test \relic{} on a variety of prediction and reinforcement learning problems. 
First, we evaluate the quality of representations pretrained on ImageNet with a special focus on robustness and out-of-distribution generalization.
\relic{} performs competitively with current state-of-the-art methods on ImageNet, while significantly outperforming competing methods on robustness and out-of-distribution generalization of the learned representations when tested on corrupted ImageNet (ImageNet-C \citep{hendrycks2019robustness}) and a version of ImageNet that consist of different renditions of the same classes (ImageNet-R \citep{hendrycks2020many}).
In terms of robustness, \relic{} also significantly outperforms the supervised baseline with an absolute reduction of $4.9\%$ in error.
Unlike much prior work that specifically focuses on computer vision tasks, we test \relic{} for representation learning in the context of reinforcement learning on the Atari suite \citep{Bellemare2013TheAL}.
There we find that \relic{} significantly outperforms competing methods and achieves above human-level performance on $51$ out of $57$ games.

\paragraph{Contributions.}
\begin{itemize}
    \item We formalize problem of self-supervised representation learning using causality and propose to more effectively leverage data augmentations through invariant prediction.
    
    \item We propose a new self-supervised objective, REpresentation Learning with Invariance Causal mechanisms (\relic), that enforces invariant prediction through an explicit regularizer and show improved generalization guarantees. 

    \item We generalize contrastive learning using refinements and show that learning on refinements is a sufficient condition for learning useful representations; this provides an alternative explanation to MI for the success of contrastive methods.
    
\end{itemize}

\section{Representation Learning via Invariant Causal Mechanisms} \label{sec:framework}

\paragraph{Problem setting.} Let $X$ denote the unlabelled observed data and $\mathcal{Y} = \{Y_t\}_{t=1}^{T}$ be a set of unknown tasks with $Y_t$ denoting the targets for task $t$. 
The tasks $\{Y_t\}_{t=1}^{T}$ can represent both a multi-environment as well as a multi-task setup.
Our goal is to pretrain with unsupervised data a representation $f(X)$ that will be useful for solving the downstream tasks $\mathcal{Y}$.

\paragraph{Causal interpretation.} To effectively leverage common assumptions and intuitions about data generation of the unknown downstream tasks for the learning algorithm, we propose to formalize them using a causal graph. 
We start from the following assumptions: a) the data is generated from \emph{content} and \emph{style} variables, with b) only content (and not style) being relevant for the unknown downstream tasks and c) content and style are independent, i.e. style changes are content-preserving.
For example, when classifying dogs against giraffes from images, different parts of the animals constitute content, while style could be, for example, background, lighting conditions and camera lens characteristics.
By assumption, content is a good representation of the data for downstream tasks and we therefore cast the goal of representation learning as estimating content.
In the following, we compactly formalize these assumptions with a causal graph\footnote{See \citep{peters2017elements} for a review of causal graphs and causality.}, see Figure \ref{fig.cgraph}.

Let $C$ and $S$ be the latent variables describing content and style. 
In Figure \ref{fig.cgraph}, the directed arrows from $C$ and $S$ to the observed data $X$ (e.g.\ images) indicate that $X$ is generated based on content and style. 
The directed arrow from $C$ to the target $Y_{t}$ (e.g.\ class labels) encodes the assumption that content directly influences the target tasks, while the absence of any directed arrow from $S$ to $Y_{t}$ indicates that style does not. 
Thus, content $C$ has all the necessary information to predict $Y_{t}$.
The absence of any directed path between $C$ and $S$ in Figure \ref{fig.cgraph} encodes the intuition that these variables are independent, i.e. $C \bigCI S$.

\begin{figure}[t]
\begin{subfigure}[b]{.40\textwidth}
\begin{tikzpicture}[scale=1.25, >=stealth]
\input{tikz_style.tex}
    \node[lat] at (0,0) (style) {$S$};
    \node[lat] at (1,0) (content) {$C$};
    \node[lat] at (2.5,0.5) (label_1)  {$Y_{1}$};
    \node[lat] at (2.5,-0.5) (label_T)  {$Y_{T}$};
    \node[detobs] at (0.5,-1.25) (pixels)  {$X$};
    \node[det] at (0.5,-2.2) (rep) {$f(X)$};
    \node[det] at (3.2,-2.2) (refinement) {$Y^R$};
    \node[draw,inner sep=1mm,label=above:Data generation,fit=(style) (pixels) (label_1) (label_T)] {};
    \node[draw,inner sep=1mm,label=below:Representation Learning,fit=(pixels) (rep) (refinement)] {};
    \path   (style) edge [directed] (pixels)
            (content) edge [directed] (pixels)
            (content) edge [directed] (label_1)
            (content) edge [directed] (label_T)
            (label_1) -- node [midway, sloped] {$\dots$} (label_T)
            (pixels) edge [directed, dashed] (rep)
            (rep) edge [directed, dashed] (refinement)
            (label_1) edge [bend left, directed] (refinement)
            (label_T) edge [bend left, directed] (refinement);
\end{tikzpicture}%
\caption{\label{fig.cgraph}}
\end{subfigure}
\begin{subfigure}[b]{.60\textwidth}
\includegraphics[width=1\textwidth]{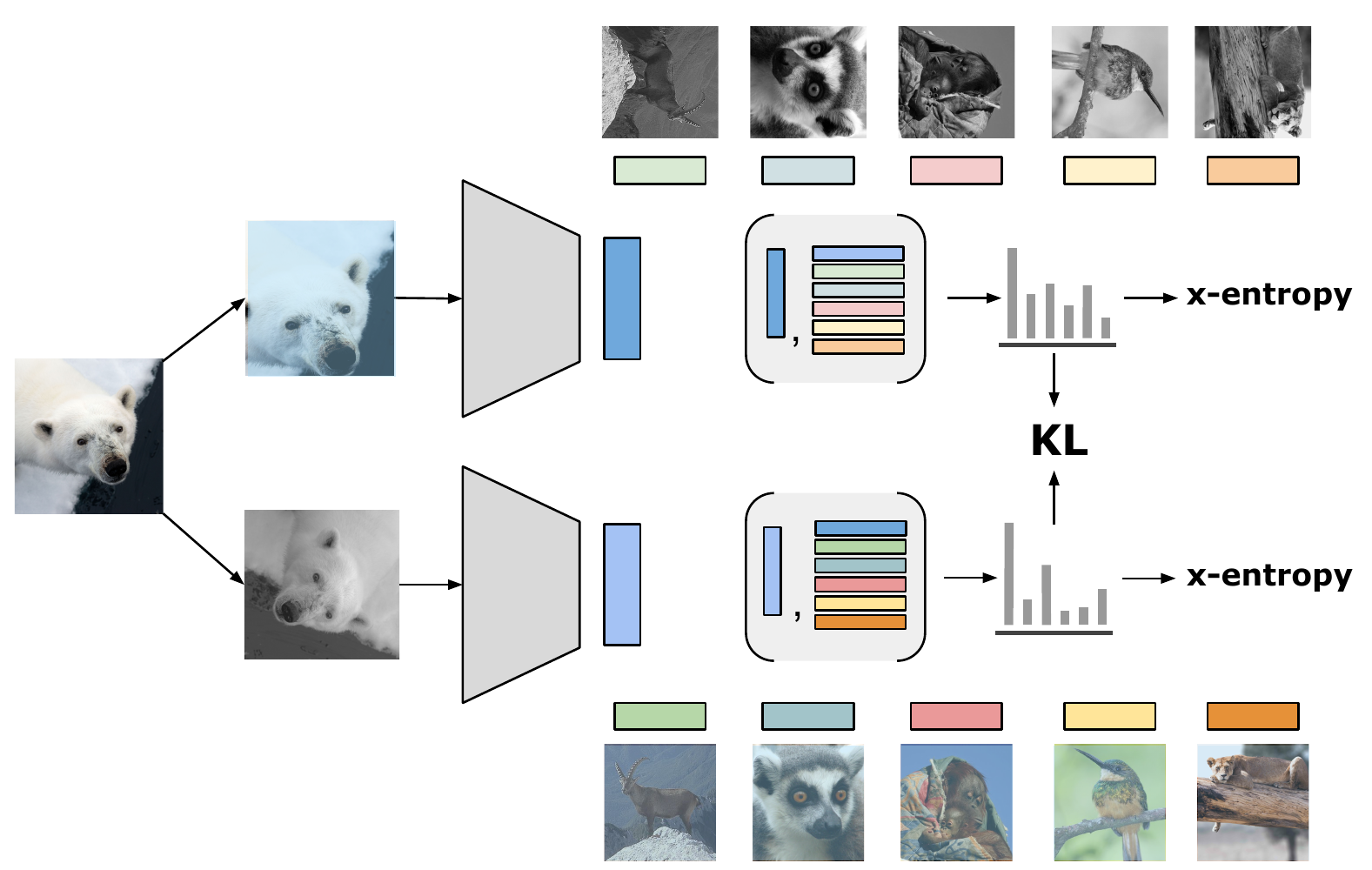}
\caption{\label{fig.objective}}
\end{subfigure}
\caption{\textbf{(a)} Causal graph formalizing assumptions about content and style of the data and the relationship between targets and proxy tasks. \textbf{(b)} \relic{} objective. KL refers to the Kullback-Leibler divergence, while x-entropy denotes cross entropy.}
\end{figure}

Using the independence of mechanisms \citep{peters2017elements}, we can conclude that under this causal model performing interventions on $S$ does not change the conditional distribution $P(Y_t\vert C)$, 
i.e. manipulating the value of $S$ does not influence this conditional distribution.
Thus, $P(Y_t\vert C)$ is invariant under changes in style $S$.
We call $C$ an \emph{invariant representation} for $Y_{t}$ under $S$, i.e.
\begin{equation}
    p^{do(S = s_{i})}(Y_{t}\,\vert\,C) = p^{do(S = s_{j})}(Y_{t}\,\vert\,C) \quad \forall \; s_{i}, s_{j}\in\mathcal{S}, \label{eq.inv}
\end{equation}
where $p^{do(S=s)}$ denotes the distribution arising from assigning $S$ the value $s$ with $\mathcal{S}$ the domain of $S$ \citep{pearl2009causality}.
Specifically, using $C$ as a representation allows for us to predict targets stably across perturbations, i.e. content $C$ is both a useful and robust representation for tasks $\mathcal{Y}$.

Since the targets $Y_{t}$ are unknown, we will construct a proxy task $Y^{R}$ in order to learn representations from unlabeled data $X$ only.
In order to learn useful representations for $Y_{t}$, we will construct proxy tasks that represents more fine-grained problems that $Y_{t}$; for a more formal treatment of proxy tasks please refer to Section \ref{sec:method}.
Further, to learn invariant representations, such as $C$, we enforce Equation \ref{eq.inv} which requires us to observe data under different style interventions, i.e. we need data that describes the same content under varying style. 
Since we do not have access to $S$, to simulate style variability we use content-preserving data augmentations  (e.g.\ rotation, grayscaling, translation, cropping \ for images).
Specifically, we utilize \emph{data augmentations as interventions on the style variable $S$}, i.e. applying data augmentation $a_{i}$ corresponds to intervening on $S$ and setting it to $s_{a_{i}}$. \footnote{Since neither content nor style are a priori known, choosing a set of augmentations implicitly defines which aspects of the data are considered style and which are content.}
Although we are not able to generate all possible styles using a fixed set of data augmentations, we will use augmentations that generate large sets of diverse styles as this allows us to learn better representations. 
Note that the heuristic of estimating similarity based on different views from contrastive learning can be interpreted as an implicit invariance constraint.

\paragraph{\relic{} objective.} 
Equation \ref{eq.inv} provides a general scheme to estimate content (c.f.\ Figure \ref{fig.cgraph}). 
We operationalize this by proposing to learn representations such that prediction of proxy targets from the representation is invariant under data augmentations.
The representation $\feat(X)$ must fulfill the following \emph{invariant prediction} criteria
\begin{equation}
    \text{(\emph{Invariant prediction})} \quad\quad p^{\text do(a_{i})}(Y^{R} \vert f(X)) = p^{\text do(a_{j})}(Y^{R} \vert f(X)) \quad \forall a_{i}, a_{j} \in \mathcal{A}.
\end{equation}
$\mathcal{A}=\{a_{1}, \dots, a_{m}\}$ is the set of data augmentations which \emph{simulate} interventions on the style variables and $p^{do(a)}$ denotes $p^{do(S = s_{a})}$.

To achieve invariant prediction, we propose to explicitly enforce invariance under augmentations through a regularizer.
This gives rise to an objective for self-supervised learning we call Representation Learning via Invariant Causal Mechanisms (\relic).
We write this objective as
\begin{flalign*}
    \expec_{X} \expec_{\substack{a_{lk}, a_{qt} \\ \sim \mathcal{A}\times\mathcal{A}}} ~~ \sum_{b\in \{a_{lk}, a_{qt}\}} \mathcal{L}_{b}(Y^R, f(X))
     ~~~ s.t. ~~~ KL\left(p^{do(a_{lk})}(Y^{R} \vert\, \feat(X)), p^{do(a_{qt})}(Y^{R} \vert\, \feat(X))\right) \leq \rho   
\end{flalign*}
where $\mathcal{L}$ is the proxy task loss and $KL$ is the Kullback-Leibler (KL) divergence.
Note that any distance measure on distributions can be used in place of the KL divergence. We explain the remaining terms in detail below.

Concretely, as proxy task we associate to every datapoint $x_{i}$ the label $y^{R}_i=i$. 
This corresponds to the instance discrimination task, commonly used in contrastive learning \citep{hadsell2006dimensionality}.
We take pairs of points $(x_{i}, x_{j})$ to compute similarity scores and use pairs of augmentations $a_{lk}=(a_{l}, a_{k})\in \mathcal{A}\times\mathcal{A}$ to perform a style intervention.
Given a batch of samples $\{x_i\}_{i=1}^N \sim \mathcal{D}$, we use
\begin{equation*}
  p^{do(a_{lk})}(Y^{R} = j \,\vert \, \feat(x_{i})) \propto \exp\left(\phi(\feat(x_{i}^{a_{l}}),\target(x_{j}^{a_{k}})) / \tau \right).
\label{eq.prob_model_intervention}
\end{equation*}
with $x^{a}$ data augmented with $a$ and $\tau$ a softmax temperature parameter. 
We  encode $\feat$ using a neural network and choose $\target$ to be related to $\feat$, e.g. $\target=\feat$ or as a network with an exponential moving average of the weights of $\feat$ (e.g.\ target networks similar to \citep{grill2020bootstrap}). 
To compare representations we use the function $\phi(\feat(x_{i}),\target(x_{j}))=\langle \critic(\feat(x_i)), \critic(\target(x_j)) \rangle$ where $\critic$ is a fully-connected neural network often called the critic. 

Combining these pieces, we learn representations by minimizing the following objective over the full set of data $x_{i}\in\mathcal{D}$ and augmentations $a_{lk}\in\mathcal{A}\times\mathcal{A}$
\begin{align}
   -\sum_{i=1}^{N}\sum_{a_{lk}} \log \frac{
    \exp\left(\phi(\feat(x_i^{a_l}), \target(x_i^{a_k}))/ \tau \right)
    }
    {
    \sum_{m=1}^M \exp\left( \phi(\feat(x_i^{a_l}), \target(x_m^{a_k}) ) / \tau\right) 
    } 
    + \alpha\sum_{a_{lk}, a_{qt}} KL(p^{do(a_{lk})}, p^{do(a_{qt})})
    \label{eq:loss}
\end{align}
with $M$ the number of points we use to construct the contrast set and $\alpha$ the weighting of the invariance penalty.
We used the shorthand $p^{do(a)}$ for $p^{do(a)}(Y^{R} = j \,\vert \, \feat(x_{i}))$.
With appropriate choices for $\phi$, $g$, $f$ and $h$ above, Equation \ref{eq:loss} recovers many recent state-of-the-art methods (c.f.\ Table \ref{tab:methods} in Section \ref{sec:relationship}). Figure~\ref{fig.objective} presents a schematic of the \relic{} objective. 

The explicit invariance penalty encourages the within-class distances (for a downstream task of interest) of the representations learned by \relic{} to be tightly concentrated. We show this empirically in Figure~\ref{fig:distances} and theoretically in Appendix~\ref{sec:gen-theory_main}. 
In the following section we provide theoretical justification for using an instance discrimination-based contrastive loss using a causal perspective. 
We also show (cf. Theorem \ref{thm.condition} below) that minimizing the contrastive loss alone (i.e. $\alpha=0$) does not guarantee generalization.
Instead, invariance across augmentations must be explicitly enforced. 

\section{Generalizing Contrastive Learning}
\label{sec:method}

\paragraph{Learning with refinements.} \label{sec:refinements}
In contrastive learning, the task of instance discrimination, i.e. classifying the dataset $\{(x_i, y_i^{R} = i) \vert x_{i}\in \mathcal{D}\}$, is used as the proxy task.
To better understand contrastive learning and motivate this proxy task, we generalize instance discrimination using the causal concept of \emph{refinements} \citep{chalupka2014visual}.
Intuitively, a refinement of one problem is another more fine-grained problem.
If task $Y_{t}$ is to classify cats against dogs, then a refinement of $Y_{t}$ is the task of classifying cats and dogs into their individual breeds.
See Figure \ref{fig:refinement} for a further visual example. 
 For any set of tasks, there exist many different refinements. 
However, the most fine-grained refinement corresponds exactly to classifying the dataset $\{(x_i, y_i^{R} = i) \vert x_{i}\in \mathcal{D}\}$.
Thus, the instance discrimination task used in contrastive learning is a specific type of refinement.
For a definition and formal treatment of refinements please refer to Appendix~\ref{sec:refinements_app}.

Let $Y^{R}$ be targets of a proxy task that is a refinement for all tasks in $\mathcal{Y}$.
Leveraging causal tools, we connect learning on refinements to learning on downstream tasks.
Specifically, we provide a theoretical justification for exchanging unknown downstream tasks with these specially constructed proxy tasks. 
We show that if $\feat(X)$ is an invariant representation for $Y^{R}$ under changes in style $S$, then $\feat(X)$ is also an invariant representation for tasks in $\mathcal{Y}$ under changes in style $S$.
Thus by enforcing invariance under style interventions on a refinement, we learn representations that generalize to downstream tasks.\footnote{Note that since refinements are more fine-grained that the original task, if a representation captures a refinement then it also captures the downstream tasks as strictly more information is needed to solve the refinement.}
This is summarized in the following theorem. 
\begin{theorem}
    Let $\mathcal{Y}=\{Y_t\}_{t=1}^{T}$ be a family of downstream tasks. Let $Y^{R}$ be a refinement for all tasks in $\mathcal{Y}$. 
    If $\feat(X)$ is an invariant representation for $Y^{R}$ under style interventions $S$, then $\feat(X)$ is an invariant representation for all tasks in $\mathcal{Y}$ under style interventions $S$, i.e.
    \begin{equation}
    \label{eq:condition}
    p^{do(s_{i})}(Y^{R}\,\vert\,\feat(X)) = p^{do(s_{j})}(Y^{R}\,\vert\,\feat(X)) \quad 
    \Rightarrow \quad 
    p^{do(s_{i})}(Y_{t}\,\vert\,\feat(X)) = p^{do(s_{j})}(Y_{t}\,\vert\,\feat(X))
    \end{equation}
for all $s_{i}, s_{j}\in\mathcal{S}$ with $p^{do(s_{i})}=p^{do(S = s_{i})}$. Thus, $\feat(X)$ is a representation that generalizes to $\mathcal{Y}$.
\label{thm.condition}
\end{theorem}

Theorem \ref{thm.condition} states that if $Y^{R}$ is a refinement of $\mathcal{Y}$ then learning a representation on $Y^{R}$ is a \emph{sufficient} condition for this representation to be useful on $\mathcal{Y}$. 
For a formal exposition of these points and accompanying proofs, please refer to Appendix \ref{sec:refinements_app}.
Recall that the instance discrimination proxy task is the most fine-grained refinement, and so the left hand side of \ref{eq:condition} is satisfied for any downstream task satisfying the stated assumptions of the theorem.

We generalize contrastive learning through refinements and connect representations learned on refinements and downstream tasks in Theorem \ref{thm.condition}. Thus, using causality we provide an alternative explanation to mutual information for the success of contrastive learning.
Note that our methodology of refinements is not limited to instance discrimination tasks and is thus more general than currently used contrastive losses. 
Real world data often includes rich sources of metadata which can be used to guide the construction of refinements
by grouping the data according to any available meta-data. 
Note that the coarser we can create a refinement, the more data efficient we can expect to be when learning representations for downstream tasks.
Further, we can also expect to require less supervised data to finetune the representation.

\section{Related Work} \label{sec:related}

{\bf Contrastive objectives and mutual information maximization.}
Many recent approaches to self-supervised learning are rooted in the well-established idea of maximizing mutual information (MI), e.g. Contrastive Predictive Coding (CPC) \citep{oord2018representation, henaff2019data}, Deep InfoMax (DIM) \citep{hjelm2018learning} and Augmented Multiscale DIM (AMDIM) \citep{bachman2019learning}. These methods are based on noise contrastive estimation (NCE) \citep{gutmann2010noise} which, under specific conditions, can be viewed as a bound on MI \citep{poole2019variational}. The resulting objective functions are commonly referred to as InfoNCE.

The precise role played by mutual information maximization in self-supervised learning is subject to some debate. \citep{tschannen2019mutual} argue that the performance on downstream tasks is not correlated with the achieved bound on MI, but may be more tightly correlated with encoder architecture and capacity. 
Importantly, InfoNCE objectives require custom architectures to ensure the network does not converge to non-informative solutions thus precluding the use of standard architectures. 
Recently, several works \citep{he2019momentum, chen2020simple} successfully combined contrastive estimation with a standard ResNet-50 architecture.
In particular, SimCLR \citep{chen2020simple} relies on a set of \emph{strong} augmentations\footnote{The set of augmentations includes Gaussian blurring, various colour distortions, flips and random cropping.}, while \citep{he2019momentum} uses a memory bank. Inspired by target networks in reinforcement learning, \citep{grill2020bootstrap} proposed BYOL: an algorithm for self-supervised learning which remarkably does not use a contrastive objective. Although theoretical explanation for the good performance of BYOL is presently missing, interestingly the objective, an $\ell_2$ distance between two different embeddings of the input data resembles the $\ell_2$ form of our regularizer proposed in Equation \ref{eq:task_euclidean} in Appendix \ref{sec:gen-theory_main}. 

\begin{wrapfigure}{r}{0.27\textwidth}
\vspace{-0.5cm}
  \begin{center}
    \includegraphics[width=0.27\textwidth]{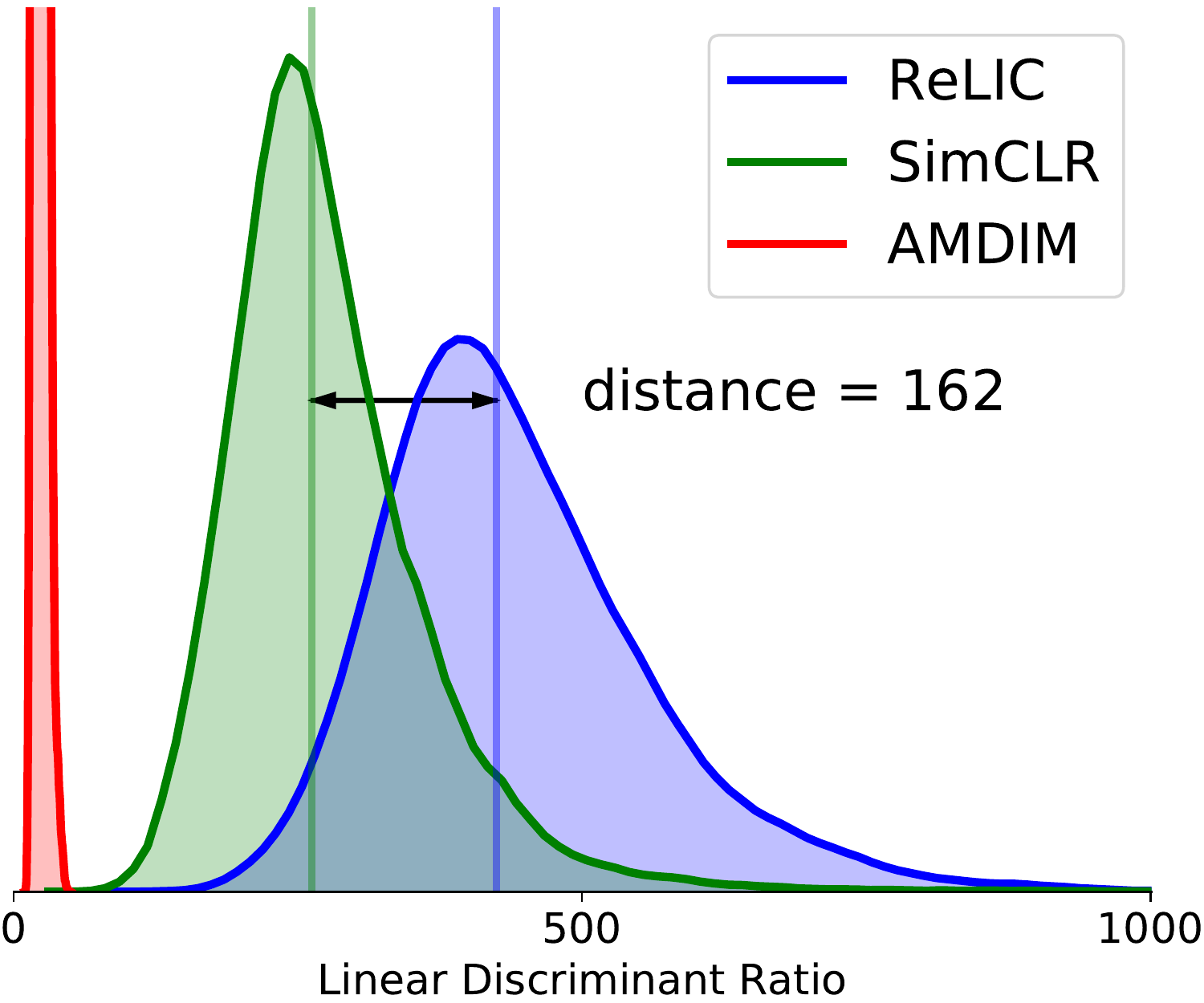}
  \end{center}
    \caption{Distribution of the linear discriminant ratio ($F_{\text{LDA}}$, see text) of $\feat$ for \relic{}, SimCLR and AMDIM ($y$-axis clipped to aid visualization). \label{fig:distances}}  
  \vspace{-0.3cm}
\end{wrapfigure}

Recently, \citep{saunshi2019theoretical} proposed a learning theoretic framework to analyze the performance of contrastive objectives. 
However, without strong assumptions on intra-class concentration they note that contrastive objectives are fundamentally limited in the representations they are able to learn. 
\relic{} explicitly enforces intra-class concentration via the invariance regularizer, ensuring that it generalizes under weaker assumptions. 
Unlike \citep{saunshi2019theoretical} which do not discuss augmentations, we incorporate augmentations into our theoretical explanation of contrastive methods. 

The reasons for the improvement in performance from AMDIM through to SimCLR and BYOL are not easily explained by either the MI maximization or the learning theoretic viewpoint. 
Further, it is not clear why relatively minor architectural differences between the methods result in significant differences in performance nor is it obvious how current state-of-the-art can be improved. 
In contrast to prior art, the performance of \relic{} is explained by connections to causal theory. 
As such it gives a clear path for improving results by devising problem appropriate refinements, interventions and invariance penalties.
Furthermore, the use of invariance penalties in \relic{} as dictated by causal theory yields significantly more robust representations that generalize better than those learned with SimCLR or BYOL.

{\bf Causality and invariance.}
Recently, the notion of invariant prediction has emerged as an important operational concept in causal inference \citep{peters2016causal}. 
This idea has been used to learn classifiers which are robust against domain shifts \citep{gong2016domain}. 
Notably, \citep{heinze2017conditional} 
propose to use group structure to delineate between different environments 
where the aim is to minimize the classification loss while also ensuring that the conditional variance of the prediction function within each group remains small.
Unlike \citep{heinze2017conditional} who use supervised data and rely on having a grouping in the training data, our approach does not rely on ground-truth targets and can flexibly create groupings of the training data if none are present. 
Further, we enforce invariant prediction within the group by constraining the distance between distributions resulting from contrasting data across groups. 

\section{Experiments}
\label{sec:experiments}

We first visualize the influence of the explicit invariance constraint in \relic{} on the linear separability of the learned representations.
We then evaluate \relic{} on a number of prediction and reinforcement learning tasks for usefulness and robustness. 
For the prediction tasks, we test \relic{} after pretraining the representation in a self-supervised way on the training set of the ImageNet ILSVRC-2012 dataset \citep{Russakovsky2015ImageNetLS}.
We evaluate \relic{} in the linear evaluation setup on ImageNet and test its robustness and out-of-distribution generalization on datasets related to ImageNet. 
Unlike much prior work in contrastive learning which focuses specifically on computer vision tasks, we test \relic{} also in the context of learning representations for reinforcement learning. 
Specifically, we test \relic{} on the suite of Atari games \citep{Bellemare2013TheAL} which consists of $57$ diverse games of varying difficulty.

\paragraph{Linear evaluation.} 
In order to understand how representations learned by \relic{} differ from other methods, we compare it against those learned by AMDIM and SimCLR in terms of Fischer's \emph{linear discriminant ratio} \citep{friedman2001elements}:
$
F_{\text{LDA}} = \Vert \mu_k - \mu_{k'}\Vert^2/\sum_{i,j\in \mathcal{C}_k} \Vert\feat(x_i) - \feat(x_j)\Vert^2 
\label{eq:lda}
$
where $\mu_k = \frac{1}{|\mathcal{C}_k|} \sum_{i\in \mathcal{C}_k} \feat(x_i)$ is the mean of the representations of class $k$ and $ \mathcal{C}_k$ is the index set of that class. 
A larger $F_{\text{LDA}}$ implies that classes are more easily separated with a linear classifier. This can be achieved by either increasing distances between classes (numerator) or shrinking within-class variance (denominator).

Figure \ref{fig:distances} shows the distribution of $F_{\text{LDA}}$ for \relic{}, SimCLR and AMDIM after training as measured on the (downsampled) ImageNet validation set. The distance between medians of \relic{} and SimCLR is 162. AMDIM is tightly concentrated close to 20. 
The invariance penalty ensures that---even though labels are \emph{a-priori} unknown---for \relic{} within-class variability of $\feat$ is concentrated leading to better linear separability between classes in the downstream task of interest.
This is reflected in the rightward shift of the distribution of $F_{\text{LDA}}$ in Figure \ref{fig:distances} for \relic{} compared with SimCLR and AMDIM which do not impose such a constraint.

\begin{wraptable}{r}{0.58\textwidth}
\caption{Accuracy (in \%) under linear evaluation on ImageNet for different self-supervised representation learning methods. Methods with * use SimCLR augmentations. Methods with $\dagger$ use custom, stronger augmentations.}
\label{table.imagenet_clean}
\begin{center}
\begin{tabular}{lccc}
\hline
Method & &Top-1 & Top-5 \\
\hline
\emph{ResNet-50 architecture} \\
\; PIRL {\footnotesize\citep{misra2020self}} && 63.6 & -  \\
\; CPC v2 {\footnotesize\citep{henaff2019data}} && 63.8 & 85.3 \\
\; CMC {\footnotesize\citep{tian2019contrastive}} && 66.2 & 87.0 \\
\; SimCLR {\footnotesize\citep{chen2020simple}} & * & 69.3 & 89.0 \\
\; SwAV {\footnotesize \citep{Caron2020UnsupervisedLO}} & * & 70.1 & - \\
\; \relic{} {\footnotesize(ours)} & * & 70.3 & 89.5 \\
\; InfoMin Aug. {\footnotesize\citep{Tian2020WhatMF}} & $\dagger$ & 73.0 &  91.1 \\
\; SwAV {\footnotesize \citep{Caron2020UnsupervisedLO}} & $\dagger$ & 75.3 & - \\
\emph{ResNet-50 with target network} \\
\; MoCo v2 {\footnotesize \citep{Chen2020ImprovedBW}}&& 71.1 &  - \\
\; BYOL {\footnotesize \citep{grill2020bootstrap}} & *& 74.3 & 91.6 \\
\; \relic{} {\footnotesize (ours)} &*& 74.8 & 92.2 \\ 
\hline
\end{tabular}
\end{center}
\end{wraptable}

Next we evaluate \relic{}'s representation by training a linear classifier of top of the fixed encoder following the procedure in \citep{Kolesnikov2019RevisitingSV,chen2020simple} and Appendix \ref{sec:linear_imagenet_app}.
In Table \ref{table.imagenet_clean}, we report top-1 and top-5 accuracy on the ImageNet test set.
Methods denoted with * use SimCLR augmentations \citep{chen2020simple}, while methods denoted $\dagger$ use custom, stronger augmentations.
Comparing methods which use SimCLR augmentations, \relic{} outperforms competing approaches on both ResNet-50 and ResNet-50 with target network.
For completeness, we report results for SwAV \citep{Caron2020UnsupervisedLO} and InfoMin \citep{Tian2020WhatMF}, but note that these methods use stronger augmentations which alone have been shown to boost performance by over $5\%$. A fair comparison between different objectives can only be achieved under the same architecture and the same set of augmentations.

\paragraph{Robustness and generalization.}
We evaluate robustness and out-of-distribution generalization of \relic{}'s representation on datasets Imagenet-C \citep{hendrycks2019robustness} and ImageNet-R \citep{hendrycks2020many}, respectively. 
To evaluate \relic{}'s representation, we train a linear classifier on top of the frozen representation following the procedure described in \citep{chen2020simple} and appendix \ref{sec:robust_app.eval}. 
For Imagenet-C we report the mean Corruption Error (mCE) and Corruption Errors for Noise corruptions in Table \ref{table.imagenet_c}.
\relic \, has significantly lower mCE than both the supervised ResNet-50 baseline and the unsupervised methods SimCLR and BYOL. 
Also, it has the lowest Corruption Error on 14 out of 15 corruptions when compared to SimCLR and BYOL.
Thus, we see that \relic{} learns the most robust representation. 
\relic \, also outperforms SimCLR and BYOL on ImageNet-R showing its superior out-of-distribution generalization ability; see Table \ref{table.imagenet_r}.
For further details and results please consult \ref{sec:robust_app}.

\begin{table}[ht]
\caption{Top-1 error rates for different self-supervised representation learning methods on ImageNet-R. All models are trained only on clean ImageNet images and \relic{}$_{T}$ refers to \relic{} using a ResNet-50 with target network as in BYOL \citep{grill2020bootstrap}.}
\label{table.imagenet_r}
\begin{center}
\begin{tabular}{lccccc}
\hline
Method & Supervised & SimCLR & \relic{} (ours) & BYOL & \relic{}$_{T}$ (ours) \\
\hline
Top-1 Error (\%) & 63.9 & 81.7 & 77.4 & 77.0 & 76.2  \\
\hline
\end{tabular}
\end{center}
\end{table}

\begin{table}[ht]
\caption{Mean Corruption Error (mCE), mean relative Corruption Error (mrCE) and Corruption Errors for the ``Noise'' class of corruptions (Gaussian, Shot, Impulse) on ImageNet-C. The mCE value is the average across $75$ different corruptions. Methods are trained only on clean ImageNet images.}
\label{table.imagenet_c}
\begin{center}
\begin{tabular}{lccccc}
\hline
\multicolumn{1}{c}{Method} & \multicolumn{1}{c}{mCE}  & \multicolumn{1}{c}{mrCE} & \multicolumn{1}{c}{Gaussian} & \multicolumn{1}{c}{Shot} & \multicolumn{1}{c}{Impulse} \\
\hline
Supervised &  76.7 & 105.0 & 80.0 & 82.0 & 83.0\\
\emph{ResNet-50 architecture:} \\
\quad SimCLR  &  87.5 & 111.9 & 79.4& 81.9& 89.6\\
\quad ReLIC (ours) &    76.4 & {\bf 87.7} & 67.8& 70.7& 77.0 \\
\emph{ResNet-50 with target network:} & & & \\
\quad BYOL  &  72.3 & 90.0 & 65.9& 68.4& 73.7 \\ 
\quad ReLIC (ours) &  {\bf 70.8} & 88.4 & {\bf 63.6} & {\bf 65.7} & {\bf 69.2}\\ 
\hline
\end{tabular}
\end{center}
\end{table}

\paragraph{Reinforcement Learning.}
Much prior work in contrastive learning has focused specifically on computer vision tasks. In order to compare these approaches in a different domain, we investigate representation learning in the context of reinforcement learning. 
We compare \relic{} as an auxiliary loss against other state of the art self-supervised losses on an agent trained on 57 Atari games. 
Using human normalized scores as a metric, we use the original architecture and hyperparameters of the R2D2 agent \citep{kapturowski2019} and supplement it with a second encoder trained with a given representation learning loss. 
When auxiliary losses are present, the Q-Network takes the output of the second encoder as an input. 
The Q-Network and the encoder are trained with separate optimizers. 
For the augmentation baseline, the Q-Network takes two identical encoders trained end-to-end. 
Table \ref{table.rl} shows a comparison between \relic{}, SimCLR, BYOL, CURL \citep{srinivas2020curl}, and feeding augmented observations directly to the agent \citep{kostrikov2020image}. 
We find that \relic{} has a significant advantage over competing self-supervised methods, performing best in 25 out of 57 games. The next best performing method, CURL performs best in 11 games. Full details are presented in Section~\ref{sec:atari_app}.
\begin{table}[ht]
\caption{Human Normalized Scores over 57 Atari Games.}
\label{table.rl}
\vspace{0.2cm}
\begin{tabular}{c|c|c|c|c|c}
\hline
 Atari Performance & \relic{} & SimCLR & CURL &  BYOL & Augmentation \\
\hline
 Capped mean & \bf{91.46} & 88.76 & 90.72 & 89.43 & 80.60 \\
 Number of superhuman games & \bf{51} & 49 & 49 & 49 & 34 \\
 Mean & \bf{3003.73} & 2086.16 & 2413.12 & 1769.43 & 503.15 \\
 Median & \bf{832.50} & 592.83 & 819.56 & 483.39 & 132.17 \\
 40\% Percentile & 356.27 & 266.07 & \bf{409.46} & 224.80 & 94.35 \\
 30\% Percentile & \bf{202.49} & 174.19 & 190.96 & 150.21 & 80.04 \\
 20\% Percentile & \bf{133.93} & 120.84 & 126.10 & 118.36 & 57.95 \\
 10\% Percentile & \bf{83.79} & 37.19 & 59.09 & 44.14 & 32.74 \\
 5\% Percentile & \bf{20.87} & 12.74 & 20.56 & 7.75 & 2.85 \\
\hline
\end{tabular}
\end{table}

\section{Conclusion}

In this work we have analyzed self-supervised learning using a causal framework. 
Using a causal graph, we have formalized the problem of self-supervised representation learning and derived properties of the optimal representation.
We have shown that representations need to be invariant predictors of proxy targets under interventions on features that are only correlated, but not causally related to the downstream tasks.
We have leveraged data augmentations to simulate these interventions and have proposed to explicitly enforce this invariance constraint.
Based on this, we have proposed a new self-supervised objective, Representation Learning via Invariant Causal Mechanisms (\relic{}), that enforces invariant prediction of proxy targets across augmentations using an invariance regularizer.
Further, we have generalized contrastive methods using the concept of refinements and have shown that learning a representation on refinements using the principle of invariant prediction is a sufficient condition for these representations to generalize to downstream tasks.
With this, we have provided an alternative explanation to mutual information for the success of contrastive methods.
Empirically we have compared \relic{} against recent self-supervised methods on a variety of prediction and reinforcement learning tasks.
Specifically, \relic{} significantly outperforms competing methods in terms of robustness and out-of-distribution generalization of the representations it learns on ImageNet. \relic{} also significantly outperforms related self-supervised methods on the Atari suite achieving superhuman performance on $51$ out of $57$ games.
We aim to investigate the construction of more coarse-grained refinements and the empirical evaluation of different kinds of refinements in future work.

\paragraph{Acknowledgements.} We thank David Balduzzi, Melanie Rey, Christina Heinze-Deml, Ilja Kuzborskij, Ali Eslami, Jeffrey de Fauw and Josip Djolonga for invaluable discussions.

\bibliographystyle{iclr2021_conference}
\bibliography{relic}

\appendix
\onecolumn
\section{Relationship between \relic{} and other methods} \label{sec:relationship}

\begin{table}[h]
\begin{center}
\caption{The objective in \eqref{eq:loss} recovers state of the art methods depending on design choices  (``-'' denotes the identity function and ``norml.'' means $g$ is constrained to have unit norm). \label{tab:methods}}
\vspace{0.2cm}
\begin{tabular}{l|c|c|c}
\multicolumn{1}{c|}{Method}       & $\phi$                               & $g$  & Regl. \\ \hline
CPC   \citep{henaff2019data}   & $\langle g, W g \rangle$     & PixelCNN & - \\
AMDIM \citep{bachman2019learning}  & $\langle \cdot, \cdot \rangle$       & - & - \\
SimCLR \citep{chen2020simple} & $\langle g, g \rangle$ & MLP, norml. & -  \\
BYOL \citep{grill2020bootstrap} & - & $g_1$, $g_2$ 1 layer MLP, norml. & $\Vert g_1(g_2) - g_2  \Vert^2$   \\ 
\relic \; (ours)& $\langle g, g \rangle$ & MLP, norml. & \Eqref{eq:loss}  \\
\end{tabular}
\end{center}
\end{table}

\section{Distance concentration and generalization} \label{sec:gen-theory_main}

Quantifying the generalization performance of representations learned on unlabelled data is a difficult task without imposing assumptions on the underlying structure of the data and the downstream tasks of interest. The results in \citep{saunshi2019theoretical} assume a latent class structure underlying the data. The similarity of images under each (potentially overlapping) latent class $c$ is measured by a probability distribution $\mathcal{D}_c$. In the contrastive setting a positive pair of points $\{x, x^{+}\}$ is said to be sampled from a distribution $\mathbb{E}_c \mathcal{D}_c(x) \mathcal{D}_c(x^{+})$ and a negative example $x^{-}$ is sampled from the marginal distribution. The task of interest is multi-class classification using the learned representation. In our setting the augmented data points $\{x^{a_l}_i, x_i^{a_k}\}$ and $\{x^{a_l}_i, x_m^{a_k}\}_{m=1}^M$ take the roles of the pairs of positive and negative points, respectively.

In this section, under the same structural assumptions on the data as \citep{saunshi2019theoretical} we will show that a similar result holds but under weaker assumptions on the function, $f$. 

\begin{wrapfigure}{r}{0.30\textwidth}
\vspace{-0.5cm}
  \begin{center}
    \includegraphics[width=0.33\textwidth]{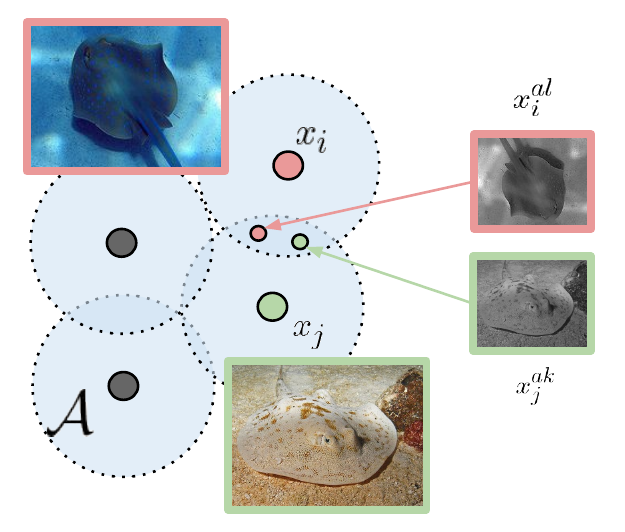}
  \end{center}
  \caption{Visual representation of invariance penalty. Shaded region denotes set of augmentations around an image.\label{fig:invariance}}
  \vspace{-0.7cm}
\end{wrapfigure}

To intuit the following results, we can view our explicit invariance constraint through the lens of distance concentration. 
Its effect can be seen intuitively in Figure~\ref{fig:invariance}. The shaded region represents the set of augmentations, $\mathcal{A}$ around an image. Depicted are two images $x_i$ and $x_j$ from the ImageNet class Stingray. 
The points $x_i^{a_{l}}$ and $x_j^{a_{k}}$ are augmentations which correspond to a region of overlap between the augmentation sets of $x_i$ and $x_j$. If the augmentations $\feat(x_i^{a_{l}})$ and $\feat(x_j^{a_{k}})$ are similar enough, encouraging $\feat(x_i)$ to be close to $\feat(x_i^{a_{l}})$ and similarly for $\feat(x_j)$ and $\feat(x_j^{a_{k}})$ indirectly encourages $\feat(x_i)$ to be close to $\feat(x_j)$. This has the effect of concentrating distances between similar images. We will make this intuition more formal in the following discussion.

Consider a modified, Euclidean distance regularized version of our objective 
\begin{flalign}
\begin{split}
    & \hat{\feat} \in \argmin_{\feat \in \mathcal{F}} ~~ \sum_{i=1}^N \sum_{a_{lk}}
    \ell(\{ \feat(x^{a_l}_i)^\top (\feat(x_i^{a_k}) - \feat(x_m^{a_k})) \}_{m=1}^M)
    \\
     & ~~~ s.t. ~~~ \Vert f(x_i) - f(x^{ak}_i) \Vert^2 \leq \rho.  
     \end{split}
     \label{eq:task_euclidean}
\end{flalign}
where $\feat\in \mathcal{F} = \{\feat: \mathcal{X}\mapsto \mathbb{R}^d ~s.t.\ ~ ||\feat||_2\leq T \}$ with $T\geq 0$. Here $\ell(v) = \log(1 + \sum_m \exp(v_m))$ is the logistic loss. For a single negative, this is equivalent to the standard \relic{} objective with an identity critic.

\paragraph{Assumptions.}\label{assn:concentration}
We require that the following assumptions hold:
    {\bf (A1)} $\hat{\feat}$ is $L$-Lipschitz and minimizes \eqref{eq:task_euclidean} such that the constraint is active and 
    {\bf (A2)} x is a bounded variable.

\begin{lemma}[Concentration\label{lem:concentration}]
If assumption (A1) holds for $\rho\leq \frac{B}{6L\kappa}$, and (A2) holds for $x$, $\hat{\feat}(x)$ is a sub-Gaussian random variable with parameter $\sigma^2_f \leq \frac{1}{\kappa} \sigma^2_x$.
\end{lemma}
See Appendix \ref{sec:gen-theory_app} for proof. This result states that the Euclidean version of our invariance regularizer has the effect of contracting the within-class variance of the data. Figure \ref{fig:distances} shows that this holds in practise for the original version of our objective in \eqref{eq:loss}. This guarantees that the following generalization result from \citep{saunshi2019theoretical} holds. For brevity we state an informal version of the Theorem with details deferred to the original publication. 

\begin{theorem}[Generalization. Adapted from Lemma B.2.\ from \citep{saunshi2019theoretical}]
Let 
$L^{\mu}_{\text{sup}}(\feat)$
be the standard $(K+1)$-wise hinge loss of the linear classification function $W^{\mu}f$ whose $c^{th}$ column is  $\mu_x = \frac{1}{|\mathcal{C}_c|} \sum_{i\in \mathcal{C}_c} \feat(x_i)$ the mean of representations corresponding to class $c$. Further, let $L^{\mu}_{\gamma(f), \text{sup}}(\feat)$ use the the hinge loss with margin $\gamma (f) = 1 + c' M \sigma_f (\sqrt{k} + \sqrt{\log \frac{M}{\epsilon}})$ with $c'$ constant and $M = \max_x \Vert f(x) \Vert$.
If $\hat{\feat}$ is the minimizer of \eqref{eq:task_euclidean} and if Assumptions \textbf{(A1)} and \textbf{(A2)} hold then with high probability
\begin{equation}
L^{\mu}_{\text{sup}}(\hat{f}) \leq \gamma(f) L^{\mu}_{\gamma(f), \text{sup}}(f) + \text{Gen}_N + \epsilon
\label{eqn:blah}
\end{equation}
\end{theorem}
Here, $\text{Gen}_N$ is a standard generalization bound which depends on the Rademacher complexity of the function class $\mathcal{F}$ and the sample size, $N$. 

For all practical purposes, the final generalization result is identical to \citep{saunshi2019theoretical} stating that $\hat{f}$---which is learned by minimizing a contrastive objective on unlabelled data---performs well on labelled data. However, this crucially depends on the intraclass concentration of the representation, that $\feat(x)$ is sub-Gaussian with parameter $\sigma^2_{\feat}$. Whereas in \citep{saunshi2019theoretical} this was assumed to hold, our Lemma \ref{lem:concentration} shows that the necessary concentration is ensured by our invariance penalty. Experimentally we see this property holds in practise (figure \ref{fig:distances}).

\section{Additional Results} \label{sec:gen-theory_app}

\begin{proof}[Proof of Lemma \ref{lem:concentration}]
Assume the data $x$ is $\sigma^2_x$-sub-Gaussian. In practise this holds since $x$ is bounded. It immediately follows that $L$-Lipschitz function $\feat(x)$ sub-Gaussian with parameter at most $L$. Now we will characterize the reduction in variance from $x$ to $\feat$. Assume there is a ball of radius $B$ around each point such that for any augmentation $x_i^{s}$ of $x_i$ $\nrm{x_i - x_i^{s}} \leq B$. By assumption (A1) we have that $\nrm{\feat(x_i) - \feat(x_i^{s})}\leq \rho$. This implies that for points $x_i$ and $x_j$ such that $\nrm{x_i - x_j} \leq 2B$, there exists a region of overlap so that $\nrm{\feat(x_i) - \feat(x_j)}\leq \nrm{\feat(x_i) - \feat(x^s_i)} + \nrm{\feat(x^s_i) - \feat(x_j)} \leq 2\rho$. 

In practise this says that there are augmentations of $x_i$ which are sufficiently similar to augmentations of $x_j$ so that their representations should be similar, thereby driving $\feat(x_i)$ and $\feat(x_j)$ to be closer.

The variance of points in $\feat$ space is
\begin{align*}
\sigma^2_{\feat} = \frac{1}{2N^2} \sum_{i}\sum_{j} \nrm{\feat(x_i) - \feat(x_j)} 
\end{align*}
The overlap $B < \nrm{x_i - x_j} \leq 2B$ induces a graph where we say $j \in \mathcal{N}(i) ~ \forall ~ j ~ \text{s.t.} ~ \nrm{x_i - x_j} \leq 2B$.
For $N$ samples we can decompose the variance as
\begin{align*}
    \sigma^2_{\feat} & = \frac{1}{2N^2} \sum_{i}\sum_{j} \nrm{\feat(x_i) - \feat(x_j)} 
    \\
    & = \frac{1}{2N^2} \sum_{i} \sum_{j\in\mathcal{N}(i)} \nrm{\feat(x_i) - \feat(x_j)} + \sum_{j'\notin\mathcal{N}(i)} \nrm{\feat(x_i) - \feat(x_{j'})} 
\end{align*}
 By smoothness of $\feat$ we always have that have $\nrm{\feat(x_i) - \feat(x_{j'})}\leq L \nrm{x_i - x_{j'}}$. By the constraint we have that $\nrm{\feat(x_i) - \feat(x_j)}\leq \frac{2\rho L}{B} \nrm{x_i - x_j} ~ \forall j \in \mathcal{N}(i) $ and for $ \delta =  \frac{2\rho L}{B} < 1$.

\paragraph{Constant proportion overlap.}
Now, assuming that for each point $i$ there is a constant proportion of the points, $0 \leq \alpha \leq 1$ in the set $\mathcal{N}(i)$ $\forall i$ we can obtain the following inequality
\begin{align}
    \sigma^2_{\feat} & = \frac{1}{2N^2} \sum_{i}\sum_{j} \nrm{\feat(x_i) - \feat(x_j)} \nonumber
    \\
    & \leq  \alpha \delta \sigma^2_x + (1-\alpha)L \sigma^2_x \nonumber
    \\
    & = (\alpha \delta + (1-\alpha)L) \sigma^2_x
    \label{eq:const_prop}
\end{align}

For $\sigma^2_{\feat} \leq \sigma^2_x$ we require $ (\alpha \delta + (1-\alpha)L)\leq 1$. Since both terms are positive we separately require $(1-\alpha)L \leq 1$:
\begin{align*}
    (1-\alpha)L & < 1
    \\
    (1-\alpha) & < \frac{1}{L}
    \\
    \alpha & > (1-\frac{1}{L})
\end{align*}
This condition makes sense since the larger $\alpha$, the fewer unconnected components in the graph.
If the above holds, we also require $\alpha\frac{2\rho L}{B} < 1 - (1-\alpha) L$ to ensure the sum is bounded above by 1. This implies $\rho <\frac{(1- (1-\alpha)L)B}{2 L \alpha}$.

However, $\alpha$ is a property of the augmentation set and not directly a user-controllable parameter so if $\alpha$ is too small or the function is not smooth enough, it might not be possible to set $\rho$ in such a way to induce contraction in $\sigma^2_{\feat}$. 

In the next section we derive a tighter concentration based on the structure of random graphs which are induced by the connectivity between data points and their augmentations.

\paragraph{Random graphs.}
 Consider the graph $G(V,E)$ induced by the constraints $(i,j) \in E ~ \forall ~ \nrm{x_i - x_j}\leq 2B$. Call $\mathcal{N}(i)$ the set of neighbours of point $i$. For $N$ points, if there is a constant probability $\alpha$ that $j\in\mathcal{N}(i)$ then $G_{N,\alpha}$ is an Erd\"{o}s-Renyi graph. 
 
 From Theorem \ref{thm:connected}, if $\alpha \geq \frac{c \log N}{N}$ for $c>1$ then with high probability, there are \emph{no} unconnected components in $G$. That is, every vertex in V is reachable from any other vertex in a finite number of steps. We can then decompose the contribution to the variance in terms of components in the graph that are adjacent and those which are reachable within a certain number of steps.

Let the degree---the shortest path---between any two points be at most $D$ we obtain the following refinement of \eqref{eq:const_prop}
\begin{align*}
    \sigma^2_{\feat} & = \frac{1}{2N^2} \sum_{i}\sum_{j} \nrm{\feat(x_i) - \feat(x_j)} 
    \\
    & \leq  \alpha \delta \sigma^2_x + (1-\alpha) D \delta \sigma^2_x
\end{align*}
From Theorem \ref{thm:diam} we have with high probability that $3\leq D \leq 4$. So for $\sigma^2_{\feat} \leq \frac{1}{\kappa} \sigma^2_x$ with $\kappa\geq 1$ we require $\rho \leq \frac{B}{2L\kappa(\alpha + 3(1-\alpha) )}\leq \frac{B}{6L\kappa}$.
\end{proof}

\begin{theorem}[Connectedness \citep{erdHos1960evolution}\label{thm:connected}]
If $p=\frac{c \log n}{n}$ where $c>1$ with high probability then the graph $G(n, p)$ has no unconnected components.
\end{theorem}

\begin{definition}[Diameter]
For a connected graph, $G(V, E)$ the diameter $\text{diam}(G) =  \max \text{dist}(v_i, v_j)$ where $\text{dist}(v_i, v_j)$ is the minimum number of edges in the path between $v_i$ and $v_j$.
\end{definition}

\begin{theorem}[Diameter of random graphs \citep{frieze2016introduction}\label{thm:diam}]
Let $d\geq2$ be a fixed positive integer. For $c>0$ and
$$
p^d n^{d-1} = \log(n^2/c)
$$
Then $\text{diam}(G_{n,p})\geq d$ with probability $\exp(-c/2)$ and $\text{diam}(G_{n,p}) \leq d+1$ with probability $1-\exp(-c/2)$.
\end{theorem}
\section{Generalizing Contrastive Learning} \label{sec:refinements_app}

\subsection{Refinements}

On the unsupervised observed data $\mathcal{D}$, any task as defined by targets $Y_{t}$ induces an equivalence relation, i.e. $Y_{t}$ partitions $\mathcal{D}$ into equivalence classes. 
It divides $\mathcal{D}$ based on values of the target, $\mathcal{D} = \{ \{ x_a \vert y_a=y_i\}_{i=1}^{M}\}$ where $\{ y_{1}, \dots, y_{M}\}$ for some $M$ is the set of target values. Here the equivalence relation associates datapoints based on the value of the target they predict. 
For example, if $\mathcal{D}$ is a set of images of cats and dogs and $Y_{t}$ denotes labels cat and dog, then $\mathcal{D}$ is partitioned into two equivalence classes corresponding to cat and dog images by $Y_{t}$. 

Intuitively, a refinement is a subdivision of an existing partition. 
For a visualization of a refinement of a set of tasks see Figure \ref{fig:refinement}.
To mathematically define refinements, we first need to introduce what it means for an equivalence relation to be finer than another equivalence relation.
\begin{definition}{\bf (Fineness).}
Let $\sim$ and $\approx$ be two equivalence relations on the set $\mathcal{D}$. If every equivalence class of $\sim$ is a subset of an equivalence class of $\approx$, we say that $\sim$ is \textit{finer} than $\approx$.
\end{definition}

Now we define what refinements.
\begin{definition}{\bf (Refinement).}
    Let $A, B$ be sets of equivalence classes induced by equivalence relations $\sim$ and $\approx$ over the set $\mathcal{D}$.
    If $\sim$ is finer than $\approx$, then we call $A$ a \textit{refinement} of B. 
\end{definition}

Furthermore, we can relate the corresponding sets of equivalence classes.
\begin{lemma}
    Let $\sim$ and $\approx$ be two equivalence relationships on the set $\mathcal{D}$ and denote the corresponding induced partitions by $A$ and $B$. If $\sim$ is finer than $\approx$, then every equivalence class of $\approx$ is a union of equivalence classes of $\sim$.
\end{lemma}
Coming back to the example of cats and dogs, let $\approx$ be the relation that associates cats with cats and dogs with dogs. Now the relation $\sim$ which associated both cats and dogs with their specific breed (e.g. poodles with other poodles) is finer than $\approx$. Note that $\sim$ partitions $\mathcal{D}$ into breeds and so we can easily generate the sets of cats and dogs (i.e. equivalence classes of $\approx$) by taking a union over all the corresponding breeds.

\begin{figure}
    \begin{center}
    \includegraphics[width=0.6\linewidth]{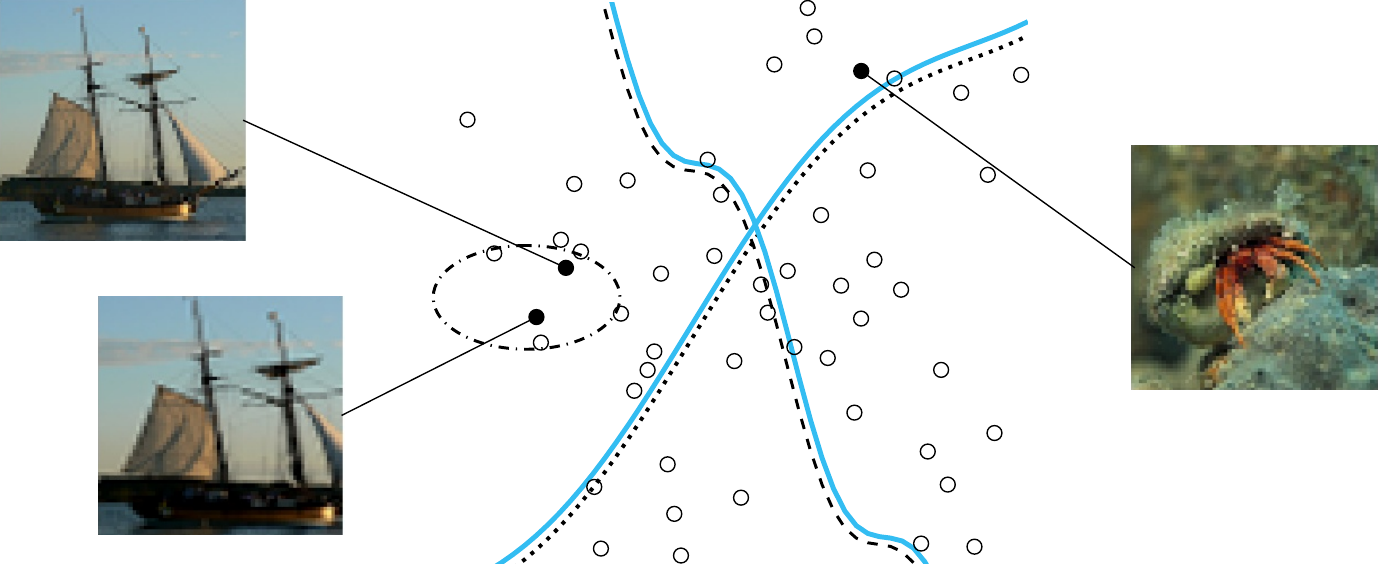}
    \end{center}
    \caption{Visualization of a refinement of a set of tasks. The tasks are to classify aquatic vs non-aquatic life and animal vs non-animal with the individual class boundaries denoted by the dashed and dotted black lines. A refinement for these tasks is a to classify aquatic animal vs aquatic non-animal vs non-aquatic animal vs non-aquatic non-animal and the class boundaries are given in teal. The ellipse indicates the set of points induced by augmenting the image of the ship.}
    \label{fig:refinement}
\end{figure}

\subsection{Proof of Theorem \ref{thm.condition}}

\begin{definition}{\bf (Invariant Representation).}
Let $X$ and $Y$ be the covariates and target, respectively. We call $f(X)$ an invariant representation for $Y$ under style $S$ if
\begin{equation}
    p^{do(S = s_{i})}(Y\,\vert\,f(X)) = p^{do(S = s_{j})}(Y\,\vert\,f(X)) \quad \forall \; s_{i}, s_{j}\in\mathcal{S}, \label{app.eq.stability}
\end{equation}
where $do(S=s)$ denotes assigning $S$ the value $s$ and $\mathcal{S}$ is the domain of $S$.
\end{definition}

\begin{oneshot}{\bf Theorem~\ref{thm.condition}.}
    Let $\mathcal{Y}=\{Y_t\}_{t=1}^{T}$ be a family of downstream tasks. Let $Y^{R}$ be a refinement for all tasks in $\mathcal{Y}$. 
    If $\feat(X)$ is an invariant representation for $Y^{R}$ under changes in style $S$, then $\feat(X)$ is an invariant representation for all tasks in $\mathcal{Y}$ under changes in style $S$, i.e.
    \begin{equation}
    p^{do(s_{i})}(Y^{R}\,\vert\,\feat(X)) = p^{do(s_{j})}(Y^{R}\,\vert\,\feat(X)) \quad 
    \Rightarrow \quad 
    p^{do(s_{i})}(Y_{t}\,\vert\,\feat(X)) = p^{do(s_{j})}(Y_{t}\,\vert\,\feat(X))
    \end{equation}
    for all $ t \in\{1, \dots, T\}$ and for all $s_{i}, s_{j}\in\mathcal{S}$ with $p^{do(s_{i})}=p^{do(S = s_{i})}$. 
Thus, $\feat(X)$ is a representation that generalizes to $\mathcal{Y}$.
\end{oneshot}

{\bf Proof.}
Let $t \in\{1, \dots, T\}$. We have
\begin{align*}
   p^{\text do(s_{i})}(Y_t \vert f(X)) 
   & = \int p^{\text do(s_{i})}(Y_t \vert Y^{R}) p^{\text do(s_{i})}(Y^{R}\vert f(X)) d\,Y^{R}
    = \int p(Y_t \vert Y^{R}) p^{\text do(s_{i})}(Y^{R}\vert f(X)) d\,Y^{R} \\
   & = \int p(Y_t \vert Y^{R}) p^{do(s_{j})}(Y^{R}\vert f(X)) d\,Y^{R} 
   = p^{do(s_{j})}(Y_t \vert f(X)).
\end{align*}
For the second and last equality, we used that the mechanism of $Y_t\vert Y^{R}$ is independent of $S$, i.e. $p^{\text do(s_{i})}(Y_t \vert Y^{R}) = p^{do(s_{j})}(Y_t \vert Y^{R})$. 
The third equality follows from the assumption that $\feat(X)$ is an invariant representation for $Y^{R}$ under changes in $S$. 
Thus, we get that $\feat(X)$ is an invariant representation for $Y_t$ under changes in $S$. Specifically, for a representation to be an invariant representation for $Y_t$ it is a \emph{sufficient condition} for it to be an invariant representation for $Y^{R}$. \qed
\section{Experimental Details} \label{sec:experiments_app}

\subsection{Image Augmentations} \label{sec:aug_app}
For pretraining the representations in \relic{}, we apply the augmentation scheme proposed in SimCLR \citep{chen2020simple} and used in \citep{grill2020bootstrap}.
This consists of the following augmentations applied in the order they are listed
\begin{itemize}
    \item random crop -- we randomly crop the image using an area randomly selected between $8\%$ and $100\%$ of the image with an logarithmically sampled aspect ration between $3/4$ and $4/3$. After this, we resize the patch to $224\times224$;
    \item random horizontal flip;
    \item color jittering -- we apply in random order perturbations to brightness, contrast, saturation and hue of the image by shifting them by a random uniform offset;
    \item grayscale -- we randomly apply grayscaling;
    \item Gaussian blurring -- we blur the image using a $23\times23$ square Gaussian kernel with standard deviation uniformly sampled in $[0.1, 0.2]$;
    \item solarization -- we transform all the pixels with $x\rightarrow x* 1_{\{x<0.5\}} + (1-x) * 1_{\{x\geq0.5\}}$.
\end{itemize}
We use the same parameters for the augmentations and probabilities of applying individual augmentations as SimCLR \citep{chen2020simple}.
After applying augmentations, we normalize the images with the mean and standard deviation computed on ImageNet across the color channels.

\subsection{Architecture}
We test \relic{} on two different architectures -- ResNet-50 \citep{he2016deep} and ResNet-50 with target network as in \citep{grill2020bootstrap}.
For ResNet-50, we use version 1 with post-activation.
We take the representation to be the output of the final average pooling layer, which is of dimension $2048$.
As in SimCLR \citep{chen2020simple}, we use a critic network to project the representation to a lower dimensional space with a multi-layer perceptron (MLP).
When using ResNet-50 as encoder, we treat the parameters of the MLP (e.g. depth and width) as hyperparameters and sweep over them.
This MLP has batch normalization \citep{Ioffe2015BatchNA} after every layer, rectified linear activations (ReLU) \citep{Nair2010RectifiedLU}.
We used a 4 layer MLP with widths $[4096, 2048, 1024, 512]$ and output size $128$ with ResNet-50.
When using a ResNet-50 with target networks as in \citep{grill2020bootstrap}, we exactly follow their architecture settings.

\subsection{Optimization}
We use a batch size of 4096 and the LARS optimizer \citep{you2017large} with a cosine decay learning rate schedule \citep{Loshchilov2017SGDRSG} for $1000$ epochs with $10$ epochs for warm-up. 
We exclude the biases and batch normalization parameters from LARS adaptation.
We use as the base learning rate $0.3$ for ResNet-50 and $0.2$ for ResNet-50 with target network.
We scale this learning rate by batch size$/256$ and use a global weight decay parameter of
$1.5*10^{-6}$ and exclude the biases and batch normalization parameters.
For the target network, we follow the approach of BYOL \citep{grill2020bootstrap} and start the exponential moving average parameter $\tau$ at $\tau_{base} = 0.996$ and increase it to one during training via $\tau = 1 - (1 - \tau_{base})(\cos(\pi k/K) + 1)/2$ with k the current training
step and K the maximum number of training steps.
 
\subsection{Evaluation on ImageNet} \label{sec:linear_imagenet_app}

We follow the standard linear evaluation protocol on ImageNet as in \citep{Kolesnikov2019RevisitingSV,chen2020simple,grill2020bootstrap}.
We train a linear classifier on top of the fixed representation, i.e. we do not update the network parameters or the batch statistics.
For training, we randomly crop and resize images to $224\times224$, and randomly horizontally flip the images after that.
For testing, the images are resized to $256$ pixels along the shorter dimension with bicubic resampling after which we take a center crop of size $224\times224$.
Both for training and testing, the images are normalized by substracting the mean and standard deviations across the color channels computed on ImageNet after the augmentations.
We use Stochastic Gradient Descent with a Nestorov momentum of $0.9$ and train for $80$ epochs with a batch size of $1024$.
We do not use any regularization techniques, e.g. weight decay.

\subsection{Robustness and Generalization}  \label{sec:robust_app}

\subsubsection{Dataset Details} \label{sec:robust_app.data}

\paragraph{ImageNet-C.} 
The ImageNet-C dataset \citep{hendrycks2019robustness} consists of $15$ different types of corruptions
from the noise, blur, weather, and digital categories applied to the validation images of ImageNet.
This dataset is used for measuring semantic robustness.
Figure \ref{corruption_types} visualizes the corruption types.
Each type of corruption has $5$ levels of severity, i.e. there are $75$ distinct corruptions in the dataset.
In Figure \ref{severities}, we display the Impulse noise corruption for $5$ different severity levels. 
As can be seen, with increasing severity level the image becomes increasingly corrupted and difficult to parse.
In addition to these $75$ corruption types, there are an additional $4$ corruption types (speckle noise, gaussian blur, spatter and saturate) that are provided as a validation set.
We use these additional corruption types for selecting the best hyperparameters.
For further details on this dataset, please refer to \citep{hendrycks2019robustness}.

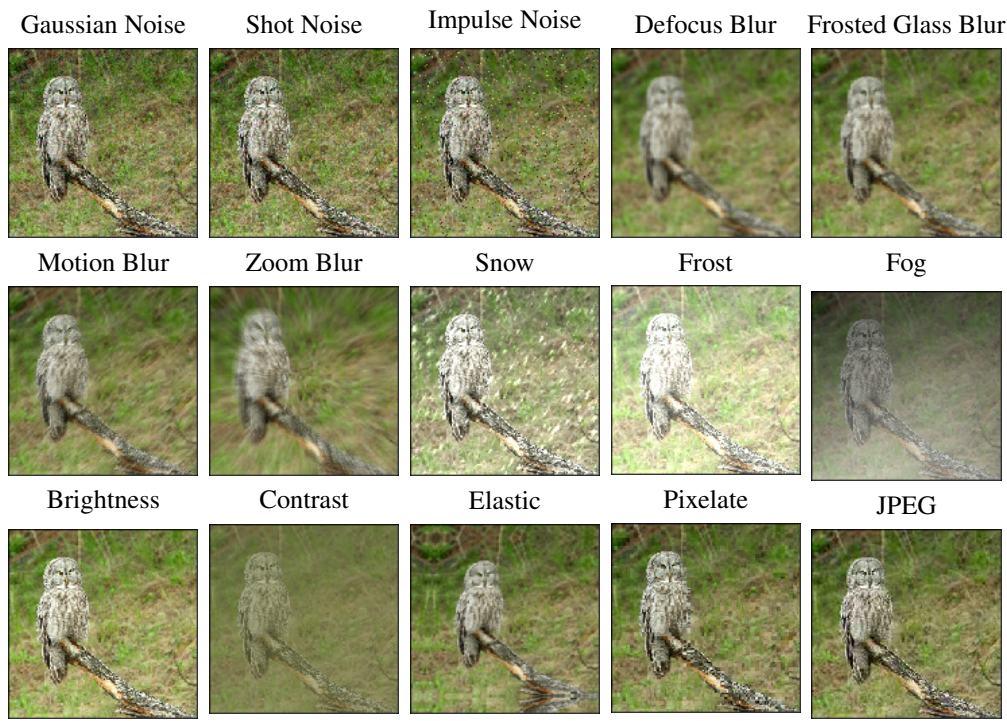
\begin{figure}[ht]
\centering
\begin{tikzpicture}[picture format/.style={inner sep=2pt,}]
\node[picture format]                   (A0)               {\includegraphics[width=0.18\textwidth]{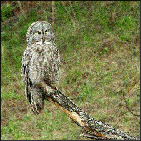}};
  \node[picture format,anchor=north west]   (A1) at (A0.north east)      {\includegraphics[width=0.18\textwidth]{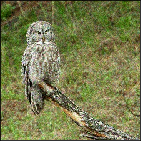}};
  \node[picture format,anchor=north west]      (A2) at (A1.north east) {\includegraphics[width=0.18\textwidth]{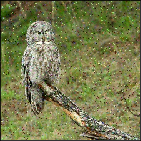}};
  \node[picture format,anchor=north west] (A3) at (A2.north east) {\includegraphics[width=0.18\textwidth]{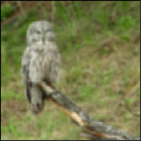}};
  \node[picture format,anchor=north west]      (A4) at (A3.north east)      {\includegraphics[width=0.18\textwidth]{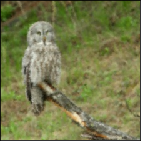}};

  \node[anchor=south] (C0) at (A0.north) {Gaussian Noise};
  \node[anchor=south] (C1) at (A1.north) {Shot Noise};
  \node[anchor=south] (C2) at (A2.north) {Impulse Noise};
  \node[anchor=south] (C3) at (A3.north) {Defocus Blur};
  \node[anchor=south] (C4) at (A4.north) {Frosted Glass Blur};

  \node[anchor=north] (C5) at (A0.south) {Motion Blur};
  \node[anchor=north] (C6) at (A1.south) {Zoom Blur};
  \node[anchor=north] (C7) at (A2.south) {Snow};
  \node[anchor=north] (C8) at (A3.south) {Frost};
  \node[anchor=north] (C9) at (A4.south) {Fog};

    \node[picture format,anchor=north] (A5) at (C5.south)
  {\includegraphics[width=0.18\textwidth]{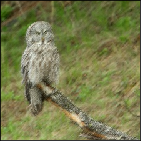}};
  \node[picture format,anchor=north]   (A6) at (C6.south)      {\includegraphics[width=0.18\textwidth]{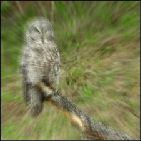}};
  \node[picture format,anchor=north]    (A7) at (C7.south) {\includegraphics[width=0.18\textwidth]{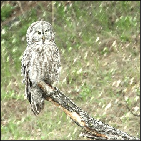}};
  \node[picture format,anchor=north] (A8) at (C8.south) {\includegraphics[width=0.18\textwidth]{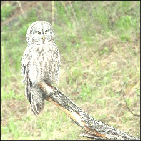}};
  \node[picture format,anchor=north] (A9) at (C9.south)      {\includegraphics[width=0.18\textwidth]{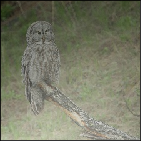}};

  \node[anchor=north] (C10) at (A5.south) {Brightness};
  \node[anchor=north] (C11) at (A6.south) {Contrast};
  \node[anchor=north] (C12) at (A7.south) {Elastic};
  \node[anchor=north] (C13) at (A8.south) {Pixelate};
  \node[anchor=north] (C14) at (A9.south) {JPEG};

   \node[picture format,anchor=north] (A10) at (C10.south)
  {\includegraphics[width=0.18\textwidth]{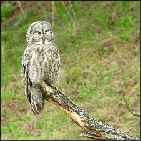}};
  \node[picture format,anchor=north]   (A11) at (C11.south)      {\includegraphics[width=0.18\textwidth]{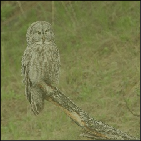}};
  \node[picture format,anchor=north]      (A12) at (C12.south) {\includegraphics[width=0.18\textwidth]{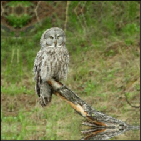}};
  \node[picture format,anchor=north] (A13) at (C13.south) {\includegraphics[width=0.18\textwidth]{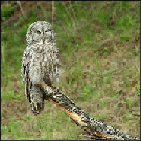}};
  \node[picture format,anchor=north]      (A14) at (C14.south)      {\includegraphics[width=0.18\textwidth]{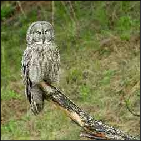}};

\end{tikzpicture}
\caption{The ImageNet-C dataset consists of 15 types of corruptions from noise, blur, weather, and digital categories. Each type of corruption has five levels of severity,
resulting in 75 distinct corruptions. See different severity levels in Figure \ref{severities}.}
\label{corruption_types}
\end{figure}

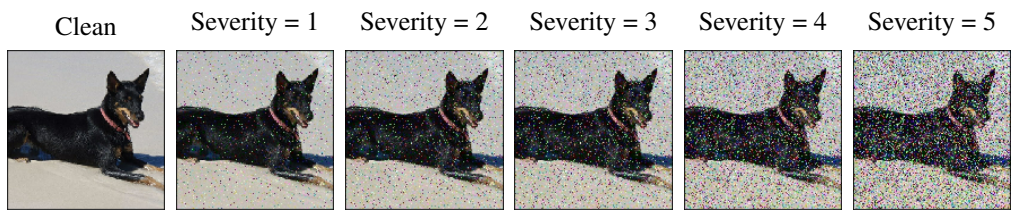
\begin{figure}
\centering
\begin{tikzpicture}[picture format/.style={inner sep=2pt,}]

\node[picture format]                   (A0)               {\includegraphics[width=0.15\textwidth]{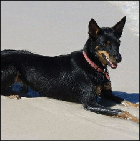}};
  \node[picture format,anchor=north west]   (A1) at (A0.north east)      {\includegraphics[width=0.15\textwidth]{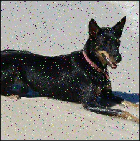}};
  \node[picture format,anchor=north west]      (A2) at (A1.north east) {\includegraphics[width=0.15\textwidth]{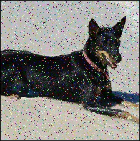}};
  \node[picture format,anchor=north west] (A3) at (A2.north east) {\includegraphics[width=0.15\textwidth]{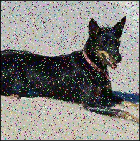}};
  \node[picture format,anchor=north west]      (A4) at (A3.north east)      {\includegraphics[width=0.15\textwidth]{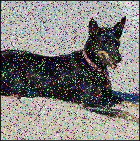}};
    \node[picture format,anchor=north west]      (A5) at (A4.north east)      {\includegraphics[width=0.15\textwidth]{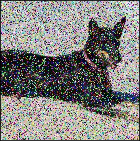}};

  \node[anchor=south] (C0) at (A0.north) {Clean};
  \node[anchor=south] (C1) at (A1.north) {Severity = 1};
  \node[anchor=south] (C2) at (A2.north) {Severity = 2};
  \node[anchor=south] (C3) at (A3.north) {Severity = 3};
  \node[anchor=south] (C4) at (A4.north) {Severity = 4};
  \node[anchor=south] (C5) at (A5.north) {Severity = 5};
\end{tikzpicture}
\caption{The $5$ different levels of severity of Impulse noise corruption available in the ImageNet-C dataset. With increasing severity the dog image is markedly corrupted.}
\label{severities}
\end{figure}

\paragraph{ImageNet-R.} 
The ImageNet-R dataset \citep{hendrycks2020many} consists of $30,000$ images depicting various
artistic renditions (e.g., paintings, sculpture, origami, cartoon) of $200$ ImageNet object classes. 
This dataset is used to measure out-of-distribution generalization to various abstract visual renditions as it emphasizes shape over texture.
The data was collected primarily from Flickr and also includes line drawings from \citep{Wang2019LearningRG}.
The images represent naturally occurring objects and have different textures and local image statistic to those of ImageNet. Figure \ref{imagenet_r} visualizes different images from the dataset.
For further details on this dataset, please refer to \citep{hendrycks2020many}.

\begin{figure}
\centering
\begin{tikzpicture}[picture format/.style={inner sep=2pt,}]

\node[picture format]                   (A0)               {\includegraphics[width=0.15\textwidth]{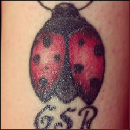}};
  \node[picture format,anchor=north west]   (A1) at (A0.north east)      {\includegraphics[width=0.15\textwidth]{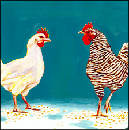}};
  \node[picture format,anchor=north west]      (A2) at (A1.north east) {\includegraphics[width=0.15\textwidth]{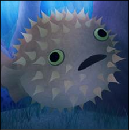}};
  \node[picture format,anchor=north west] (A3) at (A2.north east) {\includegraphics[width=0.15\textwidth]{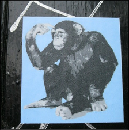}};
  \node[picture format,anchor=north west]      (A4) at (A3.north east)      {\includegraphics[width=0.15\textwidth]{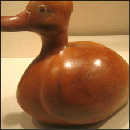}};
\end{tikzpicture}
\caption{Example images from the dataset ImageNet-R which contains $30,000$ images of $200$ ImageNet classes. 
This dataset emphasizes shape over texture and has different textures and local image statistic to those of ImageNet.}
\label{imagenet_r}
\end{figure}

\subsubsection{Evaluation} \label{sec:robust_app.eval}
To evaluate robustness and generalization of the learned representation, we follow the standard linear evaluation protocol on ImageNet as in \citep{Chen2020ImprovedBW, chen2020simple, Kolesnikov2019RevisitingSV}.
We train a linear classifier on top of the frozen representation, i.e. we do not update either the network parameters nor the batch statistics. 
During training, we augment the data by randomly cropping, resizing to $224 \times 224$ and randomly flipping the image. 
At test time, images are resized to 256 pixels along the shorter side via bicubic resampling and we take a $224\times 224$ center crop. 
Both during training and testing, after applying augmentations we normalize the color channels by subtracting the average color and dividing by the standard deviation that is computed on ImageNet.
We optimize the cross-entropy loss using Stochastic Gradient Descent with Nestorov momentum of $0.9$.
We sweep over number for epochs $\{30, 50, 60, 90\}$, learning rates $\{0.4, 0.3, 0.2, 0.1, 0.05, 0.01\}$ and batch sizes $\{1024, 2048, 4096\}$. 
We select hyperparameters on the validation set provided in ImageNet-C and report the performance on ImageNet-R and on the test set of ImageNet-C under the best validation hyperparameters.
We do not use any regularization techniques such as weight decay, gradient clipping, $tanh$ clipping or logits regularization.

\subsubsection{Robustness metrics and further results} \label{sec:robust_app.results}
Let $f$ be a classifier that has not been trained on ImageNet-C.
For each corruption type $c$ and level of severity $1\leq s \leq 5$, denote the top-1 error of this classifier as $E^{f}_{s,c}$.
Different corruption types pose different levels of difficulty. 
To make error rates across corruption types more comparable, the error rates are divided by AlexNet's errors. 
This standardized measure is the Corruption Error and is computed as 
\begin{equation*}
    CE^{f}_{c} = \left(\sum_{s=1}^{5}  E_{s,c}^{f} \right) /\left(\sum_{s=1}^{5} E_{s,c}^{AlexNet}\right)
\end{equation*}
The average error across all $15$ corruption types is called the mean Corruption Error (mCE).
Corruption Errors and mCE measure absolute robustness.

To better assess robustness, we also report the relative Corruption Error which measures relative robustness, i.e. loss in performance under corruptions. 
Denote by $E^{f}_{\text{clean}}$ the top-1 error rate for $f$ on the clean test set of ImageNet. 
The relative Corruption Error is given as
\begin{equation*}
     rCE^{f}_{c} = \sum_{s=1}^{5} \left(E_{s,c}^{f} - E^{f}_{clean}\right) / \sum_{s=1}^{5} \left(E_{s,c}^{AlexNet} - E_{clean}^{AlexNet}\right)
\end{equation*}
The mean relative Corruption Error (mrCE) is the mean of the relative Corruption Errors across all the corruption types.
For more details and intuitions about there measures please refer to \citep{hendrycks2019robustness}.

In Table \ref{table.imagenet_c_rest}, we report Corruption Errors for Blur, Weather, and Digital corruption types.
In Table \ref{table.imagenet_c_relative}, we report the relative robustness. 
As per \citep{hendrycks2019robustness}, we used the following values as the average AlexNet errors across severities, i.e. $\frac{1}{5} \sum_{s=1}^{5} E_{s, c}^{AlexNet}$, to normalize the Corruption Error values -- 
Gaussian Noise 88.6\%, Shot Noise 89.4\%, Impulse Noise 92.3\%, Defocus Blur
82.0\%, Glass Blur 82.6\%, Motion Blur 78.6\%, Zoom Blur 79.8\%, Snow 86.7\%, Frost
82.7\%, Fog 81.9\%, Brightness 56.5\%, Contrast 85.3\%, Elastic Transformation 64.6\%,
Pixelate 71.8\%, JPEG 60.7\%, Speckle Noise 84.5\%, Gaussian Blur 78.7\%, Spatter 71.8\%,
Saturate 65.8\%.

\begin{table}[ht]
\caption{Mean Corruption Error (mCE) and Corruption Error values for Blur, Weather, and Digital corruption types on ImageNet-C. 
All models are trained only using clean ImageNet images.}
\label{table.imagenet_c_rest}
\begin{center}
\resizebox{\columnwidth}{!}{%
\begin{tabular}{lc|cccc|cccc|cccc}
\hline
&  &\multicolumn{4}{c}{\bf Blur} &\multicolumn{4}{c}{\bf Weather} & \multicolumn{4}{c}{\bf Digital}\\
\hline
Method & mCE &  Defocus & Glass & Motion & Zoom & Snow & Frost & Fog & Bright & Contrast & Elastic & Pixel & JPEG \\
\hline
\quad Supervised & 76.7 & 75 &{\bf 89} & {\bf 78} & {\bf 80} & 78 & 75 & 66 & 57 & 71 & {\bf 85} & 77 & 77 \\
\emph{Using ResNet-50:} &&&&&&&&&&&&&\\
\quad SimCLR & 87.5 & 94.8& 103.3& 101.8& 101.9& 83.7& 80.6& 65.6& 71.5& 54& 106.8& 105.2& 93 \\
\quad ReLIC (ours) & 76.4 & 81.4& 96.9& 92.7& 93.2& 73.7& 71.2& 54.5& 60.2& {\bf 46.9}& 97.4& 85.5& 77.2\\
\emph{ResNet-50 with target network:} &&&&&&&&&&&&\\
\quad BYOL & 72.3 & 75& 93.6 & 86.3& 87.9& 74.3& 69.1& 48.5& 55& 48.6& 90.4& {\bf 74.3}& 73\\
\quad ReLIC (ours) & {\bf 70.8} & {\bf 73.2}& 94& 81.9 & 87 & {\bf 73.2}& {\bf 68}& {\bf 47.5} & {\bf 54.2} & 48.4& 89.5& 75.6& {\bf 71.8} \\
\hline
\end{tabular}
}
\end{center}
\end{table}

\begin{table}[h]
\caption{Mean relative Corruption Error (mrCE) and relative Corruption Error values for different corruptions and methods
on ImageNet-C. The mrCE value is the mean relative Corruption Error of the corruptions in Noise, Blur,
Weather, and Digital columns. All models are trained only using clean ImageNet images. \relic-t denotes using \relic{} with a ResNet-50+target network architecture as in BYOL \citep{grill2020bootstrap}.}
\label{table.imagenet_c_relative}
\begin{center}
\resizebox{\columnwidth}{!}{%
\begin{tabular}{lc|ccc|cccc|cccc|cccc}
\hline
&&\multicolumn{3}{c}{\bf Noise}  &\multicolumn{4}{c}{\bf Blur} &\multicolumn{4}{c}{\bf Weather} & \multicolumn{4}{c}{\bf Digital}\\
\hline
Method & mrCE & Gauss. & Shot &Impulse & Defocus & Glass & Motion & Zoom & Snow & Frost & Fog & Bright & Contrast & Elastic & Pixel & JPEG \\
\hline
\quad Supervised &105 & 104 &107 &107 &97 &{\bf 126} & {\bf 107} & {\bf 110} &101 &97 &79 &62 &89 &{\bf 146} &111 &132 \\
\emph{Using ResNet-50:} &&&&&&&&&&&&\\
\quad SimCLR & 111.9 & 88 & 92.7& 106.6& 122.2& 139.6& 140.5& 139.4& 96.9& 91.6& 59.9& 74.8& 36.7& 181.5& 158.3& 149.8 \\
\quad ReLIC (ours) & {\bf 87.7} & {\bf 67.3}& 73.1& 84.8& 96.1& 128.7& 123& 123.1& {\bf79.3}& {\bf 74.4}& 38.9& {\bf 33.4}& {\bf 24.6}& 157.5& 112& {\bf 99.8}\\
\emph{ResNet-50 with target network:} &&&&&&&&&&&&&&&\\
\quad BYOL & 90 & 72.5& 77.2& 86.7& 93& 132& 120& 122.5& 89.7& 80.2& 36.6& 41.5& 37.8& 155& {\bf 97.6}& 108.4 \\
\quad ReLIC (ours) & 88.4 & 69.1& {\bf 73}& {\bf79.2}& {\bf 90.4}& 134.2& 111.6& 121.9& 88.5& 79.2& {\bf 35.6}& 41.7& 38.5& 154.5 & 102.7& 106.8 \\
\hline
\end{tabular}
}
\end{center}
\end{table}

\subsection{Evaluation on Atari} \label{sec:atari_app}
For our experiments on Atari, we use the agent from R2D2 \citep{kapturowski2019} with standard hyperparameters noted below. We train each agent on approximately 15 billion frames and add a second encoder with the same architecture used in the Q-Network of the original agent. This second encoder is trained with a separate optimizer with only a representation learning objective. The agent then takes the output of this encoder as a given input. We use standard augmentations used in prior work \citep{kostrikov2020image} where we pad the frames on all sides with 4 pixels copied from the borders and then randomly cropping 84 windows. We randomly shift pixel intensity according to the distribution $s = 1.0 + 0.1 * \mathcal{N}'$ where $\mathcal{N}'$ is the standard Normal distribution with values clipped between -2 and 2. $s$ is then multiplied by the original image to return the augmented image.

\paragraph{\relic{} and SimCLR}
For our implementation of \relic{} and SimCLR, we do not use a critic embedding at all and utilize the last layer of the encoder for the objective. As in CURL \citep{srinivas2020curl} we utilize a target encoder for the second augmentation where we update the weights with a momentum of $.99$. We also clipped the gradients of our optimizer using a global norm ratio of 40. We report the hyperparameters in Table \ref{relic_atari}.

\begin{table}[ht]
  \caption{\relic{} and SimCLR Details}
  \centering
   \vspace{0.2cm}
  \begin{tabular}{cc}
    \toprule
    Parameter     & Value \\
    \midrule
    Normalize Inputs & True \\
    Temperature & 1.0 Constant \\
    Scaling of Embeddings & False \\
    Optimizer & Adam \\
    Learning Rate & 5e-4 \\
    Epsilon & 0.01 \\
    Beta 1 & 0.9 \\
    Beta 2 & 0.999 \\
    \bottomrule
    \label{relic_atari}
  \end{tabular}
\end{table}

\paragraph{CURL}
For CURL, we use a second encoder as noted before. With the exception of the encoder architecture and the optimizer parameters, all hyperparameters are the same as in \citep{srinivas2020curl} including the momentum value for the target network weight updates. We utilize the same architecture in the paper with a linear layer as a critic embedding for the target encoder.

\begin{table}[ht]
  \caption{CURL Details}
  \vspace{0.2cm}
  \centering
  \begin{tabular}{cc}
    \toprule
    Parameter     & Value \\
    \midrule
    Optimizer & Adam \\
    Learning Rate & 1e-3 \\
    Epsilon & 0.01 \\
    Beta 1 & 0.9 \\
    Beta 2 & 0.999 \\
    \bottomrule
    \label{curl_atari}
  \end{tabular}
\end{table}

\paragraph{BYOL}
In BYOL, we utilize two-layer perceptron networks as our predictor and projection layers. For both networks, the number of hidden units in the two layers was 1024 and 512. We use a target network update momentum of .99. The optimizer parameters are the same as in Table  \ref{relic_atari}.

\paragraph{Direct Augmentation}
We also compared against direct augmentation of the observations in the replay buffer as in DrQ \citep{kostrikov2020image}. We keep the architecture the same in this instance and use two duplicate encoders as input to the agent. In this case, the optimizer can jointly update both encoders and train them end-to-end.

\begin{table}[ht]
\tiny
  \caption{Individual Mean Episode Return on Atari.}
  \centering
   \vspace{0.2cm}
\begin{tabular}{|c|c|c|c|c|c|c|c|}
\hline
 Games & Average Human & Random & \relic{} (ours) & SimCLR & CURL & BYOL & Augmentation \\
\hline
 alien & 7127.70 & 227.80 & 8766.57 & \bf{10082.54} & 8506.48 & 9671.89 & 5201.93 \\
 amidar & 1719.50 & 5.80 & \bf{28449.26} & 28141.18 & 27213.75 & 25965.05 & 867.66 \\
 assault & 742.00 & 222.40 & \bf{92963.07} & 36109.84 & 7139.67 & 13565.20 & 1539.71 \\
 asterix & 8503.30 & 210.00 & \bf{998426.72} & 997305.51 & 661431.39 & 986307.92 & 26239.64 \\
 asteroids & 47388.70 & 719.10 & 83669.38 & 7299.90 & 76612.17 & 55936.02 & \bf{101340.17} \\
 atlantis & 29028.10 & 12850.00 & 1575940.94 & 1584392.76 & \bf{1584698.01} & 1530122.45 & 794011.79 \\
 bank heist & 753.10 & 14.20 & 1521.38 & 2467.62 & \bf{4095.29} & 1659.94 & 771.60 \\
 battle zone & 37187.50 & 2360.00 & \bf{452831.48} & 278903.14 & 287792.06 & 338695.47 & 31511.75 \\
 beam rider & 16926.50 & 363.90 & \bf{136695.24} & 98551.42 & 116794.58 & 87454.20 & 46894.14 \\
 berzerk & 2630.40 & 123.70 & \bf{146213.60} & 1301.36 & 73754.38 & 1265.21 & 73645.52 \\
 bowling & 160.70 & 23.10 & 205.09 & 193.50 & \bf{230.31} & 172.21 & 164.68 \\
 boxing & 12.10 & 0.10 & 100.00 & 100.00 & 100.00 & 100.00 & \bf{100.00} \\
 breakout & 30.50 & 1.70 & 405.05 & 404.06 & 407.14 & \bf{409.48} & 150.67 \\
 centipede & 12017.00 & 2090.90 & \bf{220886.86} & 99544.92 & 167779.11 & 146735.67 & 20152.01 \\
 chopper command & 7387.80 & 811.00 & 999900.00 & 999900.00 & \bf{999900.00} & 962003.61 & 5399.56 \\
 crazy climber & 35829.40 & 10780.50 & 272179.68 & 266870.81 & \bf{301689.62} & 210477.39 & 96538.00 \\
 defender & 18688.90 & 2874.50 & \bf{576405.57} & 522617.05 & 560816.84 & 493410.36 & 78750.19 \\
 demon attack & 1971.00 & 152.10 & 143774.79 & \bf{143786.19} & 143737.36 & 143574.86 & 821.98 \\
 double dunk & -16.40 & -18.60 & \bf{24.00} & 24.00 & 24.00 & 24.00 & 14.82 \\
 enduro & 860.50 & 0.00 & 2371.27 & 2366.19 & \bf{2373.12} & 2368.00 & 1361.66 \\
 fishing derby & -38.70 & -91.70 & 68.17 & \bf{83.00} & 72.21 & 70.11 & 19.93 \\
 freeway & 29.60 & 0.00 & 33.00 & 32.93 & \bf{33.04} & 33.00 & 32.00 \\
 frostbite & 4334.70 & 65.20 & 10156.41 & \bf{11171.49} & 3693.20 & 5793.80 & 5708.35 \\
 gopher & 2412.50 & 257.60 & \bf{123170.74} & 122368.21 & 122371.64 & 120317.04 & 43711.82 \\
 gravitar & 3351.40 & 173.00 & 4186.09 & 3601.14 & \bf{4997.87} & 4048.25 & 2014.59 \\
 hero & \bf{30826.40} & 1027.00 & 13615.35 & 13523.98 & 13620.78 & 13558.04 & 8957.00 \\
 ice hockey & 0.90 & -11.20 & 56.39 & 48.27 & 45.06 & \bf{59.70} & -2.43 \\
 jamesbond & 302.80 & 29.00 & \bf{15632.87} & 5714.62 & 10052.04 & 10099.81 & 1441.95 \\
 kangaroo & 3035.00 & 52.00 & 14342.59 & 14215.11 & 11674.19 & \bf{14471.65} & 7249.73 \\
 krull & 2665.50 & 1598.00 & \bf{137099.65} & 100426.69 & 86049.99 & 80414.04 & 16626.09 \\
 kung fu master & 22736.30 & 258.50 & \bf{230241.57} & 220076.57 & 228943.94 & 208064.38 & 64632.42 \\
 montezuma revenge & \bf{4753.30} & 0.00 & 1066.67 & 733.33 & 1072.30 & 419.54 & 26.67 \\
 ms pacman & 6951.60 & 307.30 & 13367.55 & 12053.76 & \bf{13465.80} & 12726.79 & 3238.90 \\
 name this game & 8049.00 & 2292.30 & \bf{48669.30} & 46657.55 & 47417.82 & 44848.29 & 13416.57 \\
 phoenix & 7242.60 & 761.40 & \bf{803108.37} & 253542.40 & 580969.56 & 20317.80 & 6264.39 \\
 pitfall & \bf{6463.70} & -229.40 & 0.00 & 0.00 & 0.00 & 0.00 & 0.00 \\
 pong & 14.60 & -20.70 & \bf{21.00} & 21.00 & 21.00 & 21.00 & 21.00 \\
 private eye & \bf{69571.30} & 24.90 & 10154.93 & 5115.34 & 5190.28 & 470.68 & 111.77 \\
 qbert & 13455.00 & 163.90 & \bf{353197.13} & 24340.75 & 208207.97 & 57261.24 & 11051.97 \\
 riverraid & 17118.00 & 1338.50 & \bf{23525.44} & 20400.83 & 20230.02 & 22206.57 & 10487.59 \\
 road runner & 7845.00 & 11.50 & 213173.15 & 236235.30 & 241917.98 & 238880.54 & \bf{440430.17} \\
 robotank & 11.90 & 2.20 & 97.65 & 82.60 & \bf{98.13} & 62.54 & 49.98 \\
 seaquest & 42054.70 & 68.40 & \bf{999999.00} & 999999.00 & 666700.67 & 29160.93 & 37397.26 \\
 skiing & \bf{-4336.90} & -17098.10 & -24761.06 & -23076.73 & -15497.66 & -26028.08 & -22162.91 \\
 solaris & \bf{12326.70} & 1236.30 & 4594.37 & 4571.27 & 4276.39 & 4331.03 & 4142.69 \\
 space invaders & 1668.70 & 148.00 & \bf{3625.52} & 3619.94 & 3542.48 & 3613.93 & 835.37 \\
 star gunner & 10250.00 & 664.00 & 283499.72 & \bf{289099.89} & 129720.84 & 175486.67 & 43167.07 \\
 surround & 6.50 & -10.00 & \bf{10.00} & 9.96 & 1.60 & 9.56 & -0.64 \\
 tennis & -8.30 & -23.80 & 0.00 & 0.00 & 0.00 & 0.00 & \bf{0.12} \\
 time pilot & 5229.20 & 3568.00 & 309297.74 & 92888.66 & \bf{400326.69} & 48011.44 & 14198.37 \\
 tutankham & 167.60 & 11.40 & \bf{371.17} & 306.45 & 337.61 & 285.36 & 144.30 \\
 up n down & 11693.20 & 533.40 & \bf{577256.03} & 520666.59 & 566912.89 & 552110.67 & 143512.38 \\
 venture & 1187.50 & 0.00 & 1929.53 & \bf{1945.20} & 1906.84 & 1881.76 & 733.29 \\
 video pinball & 17667.90 & 0.00 & 978292.52 & \bf{993332.08} & 932523.58 & 623223.24 & 37584.71 \\
 wizard of wor & 4756.50 & 563.50 & \bf{123513.74} & 89462.62 & 106801.20 & 68256.44 & 5940.82 \\
 yars revenge & 54576.90 & 3092.90 & 228704.52 & 99636.25 & \bf{229221.52} & 86847.75 & 48041.63 \\
 zaxxon & 9173.30 & 32.50 & \bf{120830.77} & 57379.66 & 85906.74 & 48067.61 & 23688.22 \\
\hline
\end{tabular}
\end{table}
\normalsize

\end{document}